\documentclass[10pt]{elsarticle}
\usepackage{amsmath,amssymb,graphicx}
\usepackage{url}  
\usepackage{graphicx}  
\usepackage{color}
\usepackage{booktabs} 

\usepackage{wasysym}

\newtheorem{theorem}{Theorem}
\newtheorem{lemma}{Lemma}
\newtheorem{claim}{Claim}
\newtheorem{proposition}{Proposition}
\newtheorem{definition}{Definition}

\renewcommand{\phi}{\varphi}
\renewcommand{\epsilon}{\varepsilon}

\newenvironment{proof}{\noindent{\em Proof.}}{\hfill $\boxtimes\hspace{2mm}$\linebreak}
\renewcommand{\qed}{\hfill $\boxtimes\hspace{2mm}$}

\newenvironment{proof-of-claim}{\noindent{\em Proof of Claim.}}{\hfill $\boxtimes\hspace{2mm}$\linebreak}

\usepackage{epigraph}

\renewcommand{\H}{{\sf H}}
\newcommand{\N}{{\sf N}}
\newcommand{\cN}{{\sf \overline{N}}}
\newcommand{\K}{{\sf K}}
\newcommand{\E}{{\sf E}}
\renewcommand{\S}{{\sf S}}
\renewcommand{\c}{{\sf c}}

\begin{document}

\begin{frontmatter}





\title{If You're Happy, Then You Know It: The Logic of Happiness \dots and Sadness}


\author{Sanaz Azimipour}
\address{Berlin University of Applied Sciences, Berlin, Germany}
\ead{Sanaz.a234@gmail.com}

\author{Pavel Naumov}
\address{King's College, Pennsylvania,  USA}
\ead{pgn2@cornell.edu}

\begin{abstract}
The article proposes a formal semantics of happiness and sadness modalities in imperfect information setting. It shows that these modalities are not definable through each other and gives a sound and complete axiomatization of their properties.
\end{abstract}

\end{frontmatter}

\section{Introduction}

Emotions provide motivation for many aspects of human behavior. To be able to understand and predict human actions, artificial agents must be able to identify, comprehend, and reason about human emotions. Different formal models of human emotions have been studied in AI literature. 
Doyle, Shoham, and Wellman propose a logic of relative desire~\cite{dsw91ismis}.
Lang, Van Der Torre, and Weydert introduce utilitarian desires~\cite{lvw02aamas}.
Meyer states logical principles aiming at capturing anger and fear~\cite{m04ecai}.
Steunebrink, Dastani, and Meyer expand this work to hope~\cite{sdm07aaai}.
Adam, Herzig, and Longin propose formal definitions of hope, fear, relief, disappointment, resentment, gloating, pride, shame, admiration, reproach, gratification, remorse, gratitude, and anger~\cite{ahl09synthese}.
Lorini and Schwarzentruber define regret and elation~\cite{ls11ai}.

The focus of this article is on happiness and sadness. These notions have long been studied in literature on philosophy~\cite{s11ijw,f19ijw,b10yup,a10dup}, psychology~\cite{a13routledge,sed05pid}, and economics~\cite{f10mit,bp05oup}.
Note that happiness/sadness are vague terms that have multiple meanings that overlap with several other terms such as joy/distress and elation/disappointment\footnote{For example, Merriam-Webster dictionary lists joy as synonym for happiness and happiness as synonym for joy. At the same time, thesaurus.com lists elation as synonym for joy and joy as synonym for elation.}. 

Two approaches to capturing happiness and sadness in formal logical systems have been proposed.   The first approach is based on Ortony, Clore, and Collins' definitions of joy and distress (the capitalization is original):
\begin{quote}
\dots we shall often use the terms ``joy'' and ``distress'' as convenient shorthands for the reactions of being PLEASED ABOUT A DESIRABLE EVENT and DISPLEASED ABOUT AN UNDESIRABLE EVENT, respectively.~\cite[p.88]{occ88cup} 
\end{quote}
Adam, Herzig, and Longin formalized these definitions. An agent feels joy about $\phi$ if she believes that $\phi$ is true and she desires $\phi$. An agent feels distress about $\phi$ if she believes $\phi$ is true, but she desires $\phi$ to be false~\cite{ahl09synthese}. 
Similarly, Lorini and Schwarzentruber define that an agent is elated/disappointed about $\phi$ if $\phi$ is desirable/undesirable to the agent, agent knows that $\phi$ is true, and she also knows that the others could have prevented $\phi$ from being true~\cite{ls11ai}. Although Adam, Herzig, and Longin use beliefs while Lorini and Schwarzentruber use knowledge and the latter authors also add ``could have prevented'' part, both definitions could be viewed as a variation of Ortony, Clore, and Collins' definitions of joy/distress.

Meyer suggested a very different  approach to defining these notions.
He writes ``an agent that is happy observes that its subgoals (towards certain goals) are being achieved, and is `happy' with it''. He acknowledges, however, that this definition might be capturing only one of the forms of what people mean by happiness~\cite{m04ecai}. 

Note, for example, that if Pavel, the second author of this article, receives an unexpected gift from Sanaz, the first author, then he will experience ``joy'' as defined by Ortony, Clore, and Collins. However, he will not be ``happy'' as defined by Meyer because receiving such a gift has never been among Pavel's goals\footnote{Ortony, Clore, and Collins give a similar example with an unexpected inheritance from an unknown relative.}.

In this article we adopt Ortony, Clore, and Collins' definitions, but we use terms happiness/sadness instead of joy/distress. While the cited above works suggest formal semantics and list formal properties  of happiness and sadness, none of them gives an axiomatization of these properties. In this article we propose such an axiomatization and prove its completeness. We also show that notions of happiness and sadness in our formal system are, in some sense, dual but are not definable through each other.

The rest of the article is structured as follows. First, we formally define epistemic models with preferences that serve as the foundation of our semantics of happiness and sadness. Then, we define formal syntax and semantics of our system and illustrate them with several examples. In Section~\ref{Undefinability of Sadness through Happiness section}, we show that sadness cannot be defined through happiness. In spite of this, as we show in Section~\ref{Duality of Happiness and Sadness} there is a certain duality between the properties of the happiness and sadness. We use this duality to observe that sadness can not be defined through happiness either. In Section~\ref{Axioms of Emotions}, we list the axioms of our logical system. In the section that follows, we prove its soundness. In Section~\ref{Utilitarian Emotions} and Section~\ref{Goodness-Based Emotions}, we show how  utilitarian and goodness-based approaches to desires, already existing in the literature, could be adopted to happiness and sadness. In the rest of the article we prove the completeness of our logical section. The last section concludes.

\section{Epistemic Models with Preferences}

Throughout the article we assume a fixed countable set of agents $\mathcal{A}$ and a countable set of propositional variables.
The semantics of our logical system is defined in terms of epistemic models with preferences. These models extend standard Kripke models for epistemic logic with a preference relation for each agent in set $\mathcal{A}$. 

\begin{definition}\label{epistemic model with preferences}
A tuple $(W,\{\sim_a\}_{a\in\mathcal{A}},\{\prec_a\}_{a\in\mathcal{A}},\pi)$ is called an epistemic model with preferences if
\begin{enumerate}
    \item $W$ is a set of epistemic worlds,
    \item $\sim_a$ is an ``indistinguishability'' equivalence relation on set $W$ for each agent $a\in\mathcal{A}$,
    \item $\prec_a$ is a strict partial order preference relation on set $W$ for each agent $a\in\mathcal{A}$,
    \item $\pi(p)$ is a subset of $W$ for each propositional variable $p$.
\end{enumerate}
\end{definition}
For any two sets of epistemic worlds $U,V\subseteq W$, we write $U\prec_a V$ if $u\prec_a v$ for each world $u\in U$ and each world $v\in V$.

\begin{figure}[ht]
\begin{center}
\scalebox{0.45}{\includegraphics{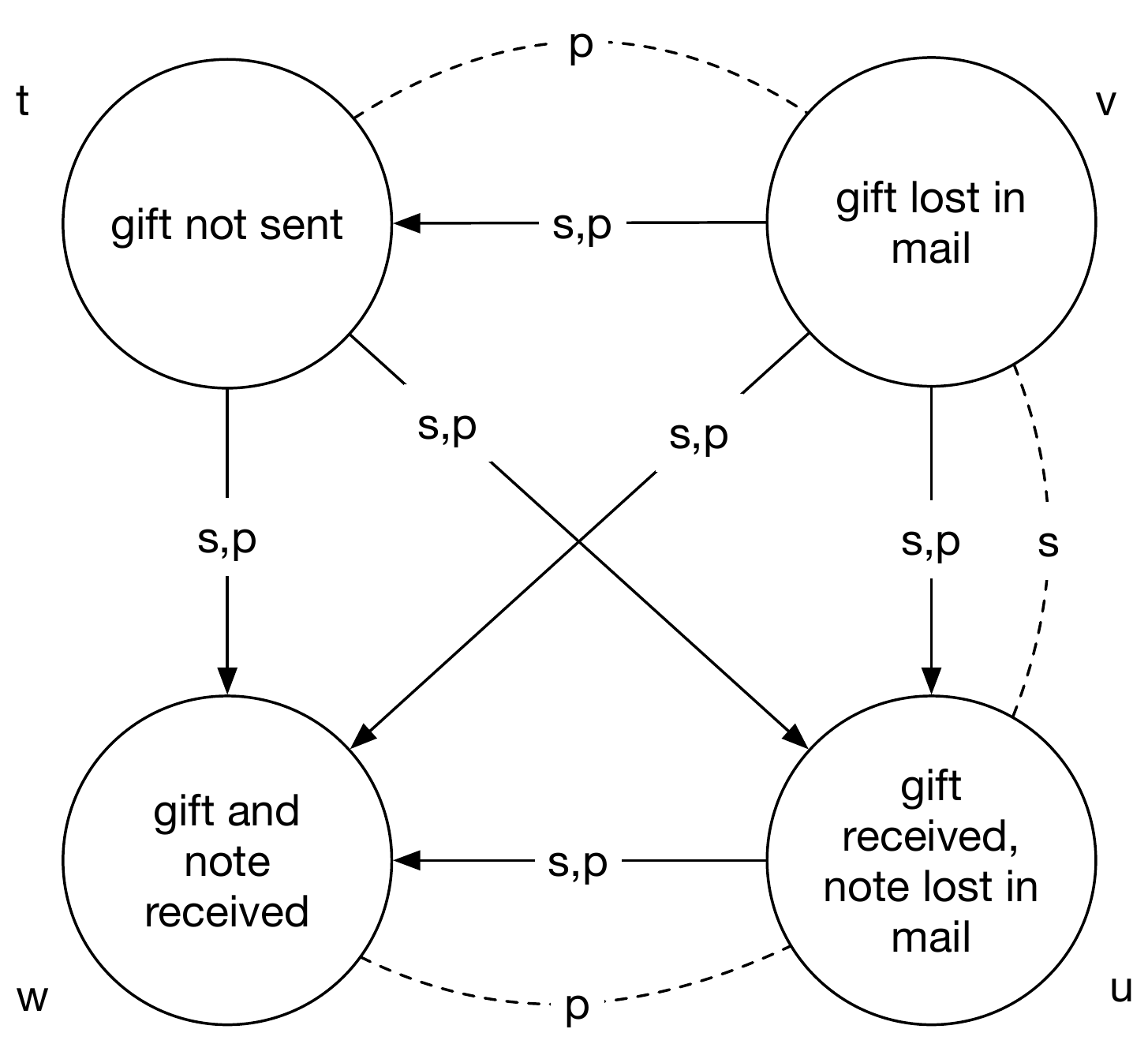}}
\caption{Gift Scenario.}\label{example gift figure}
\end{center}
\end{figure}

An example of an epistemic model with preferences is depicted in Figure~\ref{example gift figure}. It captures mentioned in the introduction scenario in which the first author of this article, Sanaz, is sending a gift to the second author, Pavel. If Pavel receives the gift, he will acknowledge it by sending back a thank you note. We assume that either the gift or the note can be lost in the mail and that there is no additional communication between the authors.

This epistemic model with preferences in Figure~\ref{example gift figure} has four worlds corresponding to four different scenarios. In world $t$ Sanaz did not send the gift. In world $v$, she sent the gift, but it was lost in the mail. In world $u$, Pavel received the gift, but his note is lost in the mail. Finally, in world $w$ Pavel received the gift and Sanaz received his note. Sanaz cannot distinguish the world $v$, in which the gift is lost, from world $u$, in which the note is lost. Pavel cannot distinguish the world $t$, in which gift is not sent, from world $v$, in which the gift is lost. He also cannot distinguish the world $u$, in which the note is lost, from world $w$, in which the note is received. In the diagram, the indistinguishability relations of Sanaz and Pavel are denoted by dashed lines labeled by $s$ and $p$ respectively. Although in general, according to Definition~\ref{epistemic model with preferences}, different agents might have different preference relations between world, in this scenario we assume that Sanaz and Pavel have the same preferences. These preferences are shown in the diagram using directed edges. For example, the directed edge from world $v$ to world $t$ labeled with $s,p$ means that they both would prefer if Sanaz does not send the gift at all to the scenario when the gift is lost in the mail. 

\section{Syntax and Semantics}
In this section we introduce the formal syntax and semantics of our logical system. Throughout the article we assume a fixed countable set of agents $\mathcal{A}$ and a fixed nonempty countable set of propositional variables. The language $\Phi$ of our logical system is defined by grammar:
$$
\phi:=p\;|\;\neg\phi\;|\;\phi\to\phi\;|\;\N\phi\;|\;\K_a\phi\;|\;\H_a\phi\;|\;\S_a\phi,
$$
where $p$ is a propositional variable and $a\in\mathcal{A}$ is an agent. We read formula $\N\phi$ as ``statement $\phi$ is true in each world'', formula $\K_a\phi$ as ``agent $a$ knows $\phi$'', formula $\H_a\phi$ as ``agent $a$ is happy about $\phi$'', and formula $\S_a\phi$ as ``agent $a$ is sad about $\phi$''. We assume that Boolean connectives conjunction $\wedge$, disjunction $\vee$, and biconditional $\leftrightarrow$ are defined through negation $\neg$ and implication $\to$ in the standard way. By $\cN\phi$ we denote formula $\neg\N\neg\phi$. We read $\cN\phi$ as ``statement $\phi$ is true in at least one of the worlds''.

\begin{definition}\label{sat}
For any world $w\in W$ of an epistemic model with preferences $(W,\{\sim_a\}_{a\in\mathcal{A}},\{\prec_a\}_{a\in\mathcal{A}},\pi)$ and any formula $\phi\in\Phi$, satisfaction relation $w\Vdash\phi$ is defined as follows:
\begin{enumerate}
    \item $w\Vdash p$, if $w\in \pi(p)$,
    \item $w\Vdash\neg\phi$, if $w\nVdash\phi$,
    \item $w\Vdash\phi\to\psi$, if $w\nVdash\phi$ or $w\Vdash\psi$,
    \item $w\Vdash\N\phi$, if $u\Vdash \phi$ for each world $u\in W$,
    \item\label{item K} $w\Vdash\K_a\phi$, if $u\Vdash \phi$ for each world $u\in W$ such that $w\sim_a u$,
    \item\label{item H} $w\Vdash \H_a\phi$, if the following three conditions are satisfied:
        \begin{enumerate}
            \item $u\Vdash \phi$ for each world $u\in W$ such that $w\sim_a u$,
            \item for any two worlds $u,u'\in W$, if $u\nVdash\phi$ and $u'\Vdash\phi$, then $u\prec_a u'$,
            \item there is a world $u\in W$ such that $u\nVdash\phi$,
        \end{enumerate}
    \item $w\Vdash \S_a\phi$, if the following three conditions are satisfied:
        \begin{enumerate}
            \item $u\Vdash \phi$ for each world $u\in W$ such that $w\sim_a u$,
            \item for any two worlds $u,u'\in W$, if $u\Vdash\phi$ and $u'\nVdash\phi$, then $u\prec_a u'$,
            \item there is a world $u\in W$ such that $u\nVdash\phi$.
        \end{enumerate}
\end{enumerate}
\end{definition}

Items 6 and 7 of Definition~\ref{sat} capture the notions of happiness and sadness studied in this article. Item 6 states that to be happy about a condition $\phi$, the agent must know that $\phi$ is true, the agent must prefer the worlds in which condition $\phi$ is true to those where it is false, and the condition $\phi$ must not be trivial. These three parts are captured by items 6(a), 6(b), and 6(c) of the above definition. Note that we require condition $\phi$ to be non-trivial to exclude an agent being happy about conditions that always hold in the model. Thus, for example, we believe that an agent cannot be happy that $2+2=4$.

Similarly, item 7 states that an agent is sad about condition $\phi$ if she knows that $\phi$ is true, she prefers worlds in which condition $\phi$ is false to those in which condition $\phi$ is true, and condition $\phi$ is not trivial. Note that being sad is different from not being happy. In fact, later in this article we show that neither of  modalities $\H$ and $\S$ is expressible through the other.

We conclude this section with a technical observation that follows from Definition~\ref{sat}.
\begin{lemma}\label{semantic substitution lemma}
For any epistemic model with preferences $(W,\{\sim_a\}_{a\in\mathcal{A}},\{\prec_a\}_{a\in\mathcal{A}},\pi)$, any agent $a\in\mathcal{A}$, and any formulae $\phi,\psi\in\Phi$, if $w\Vdash\phi$ iff $w\Vdash\psi$ for each world $w\in W$, then  
\begin{enumerate}
    \item $w\Vdash\H_a\phi$ iff $w\Vdash\H_a \psi$ for each world $w\in W$,
    \item $w\Vdash\S_a\phi$ iff $w\Vdash\S_a \psi$ for each world $w\in W$. \qed
\end{enumerate}
\end{lemma}

\section{Gift Scenario}\label{gift scenario section}

In this section we illustrate Definition~\ref{sat} using the gift scenario depicted in the diagram in Figure~\ref{example gift figure}. For the benefit of the reader, we reproduce this diagram in Figure~\ref{example gift figure repeated}.

\begin{figure}[ht]
\begin{center}
\scalebox{0.45}{\includegraphics{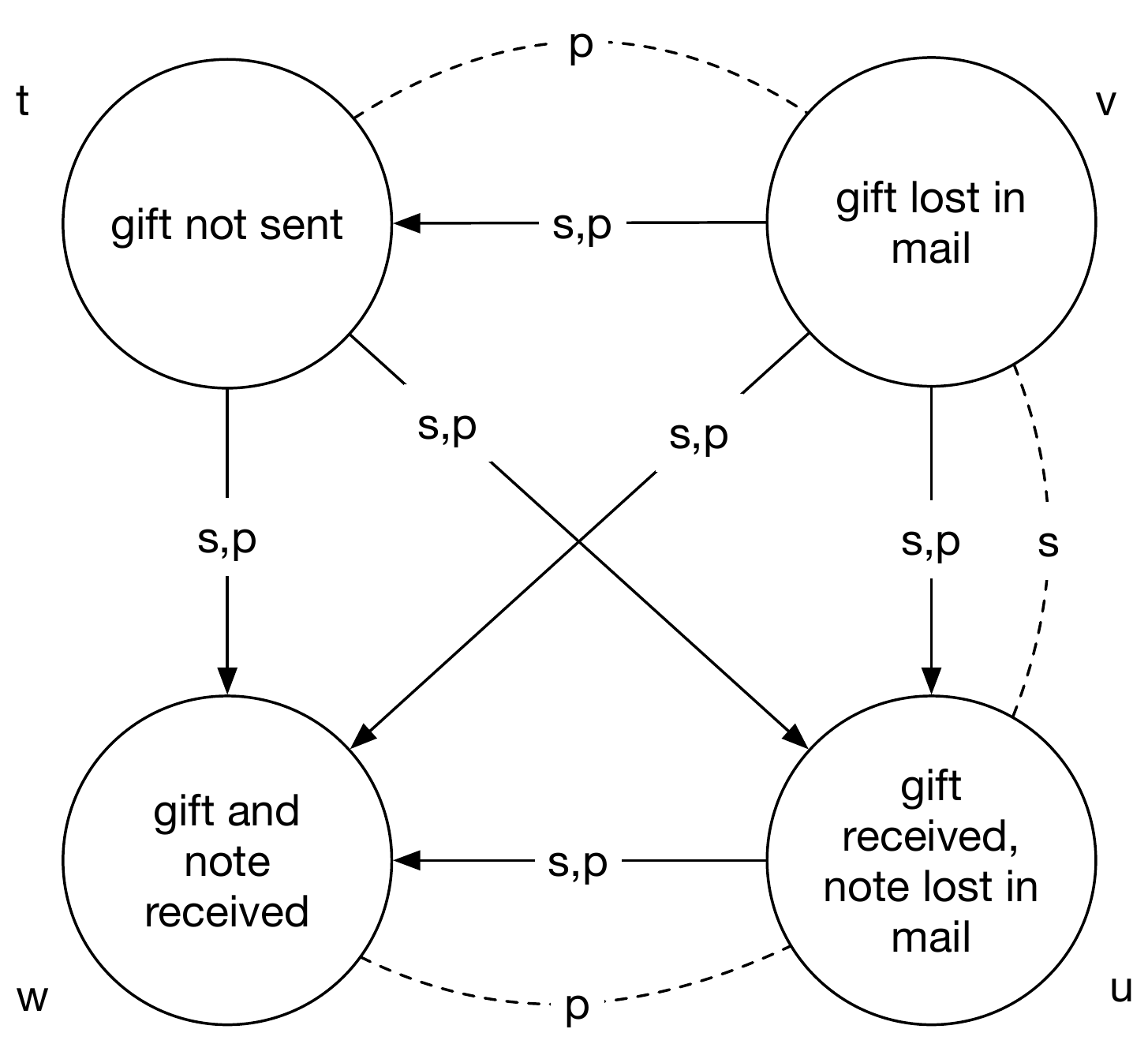}}
\caption{Gift Scenario (repeated from Figure~\ref{example gift figure}).}\label{example gift figure repeated}
\end{center}
\end{figure}


\begin{proposition}\label{Hp proposition}
$z\Vdash \H_p(\mbox{``Pavel received a gift from Sanaz''})$ iff $z\in\{w,u\}$.
\end{proposition}
\begin{proof} 
$(\Rightarrow):$ Suppose $z\notin \{w,u\}$. Thus, $z\in \{t,v\}$.  Hence, see Figure~\ref{example gift figure repeated},
$$z\nVdash \mbox{``Pavel received a gift from Sanaz''}.$$
Therefore, $z\nVdash \H_p(\mbox{``Pavel received a gift from Sanaz''})$ by item 6(a) of Definition~\ref{sat} because $z\sim_p z$.

\vspace{1mm}
\noindent $(\Leftarrow):$ Suppose $z\in \{w,u\}$. To show $z\Vdash \H_p(\mbox{``Pavel received a gift from Sanaz''})$, we will verify conditions (a), (b) and (c) of item~\ref{item H} of Definition~\ref{sat} separately:

\vspace{1mm}
\noindent{\em Condition a:} Consider any world $z'$ such that $z\sim_p z'$. To verify the condition, it suffices to show that $z'\Vdash \mbox{``Pavel received a gift from Sanaz''}$. Indeed, assumptions $z\in \{w,u\}$ and $z\sim_p z'$ imply that $z'\in \{w,u\}$, see Figure~\ref{example gift figure repeated}. Therefore, $z'\Vdash \mbox{``Pavel received a gift from Sanaz''}$, again see Figure~\ref{example gift figure repeated}.

\vspace{1mm}
\noindent{\em Condition b:} Consider any two epistemic worlds $x,y$ such that 
\begin{eqnarray}
    &&x\nVdash \mbox{``Pavel received a gift from Sanaz''},\label{sept12-1}\\
    &&y\Vdash \mbox{``Pavel received a gift from Sanaz''}.\label{sept12-2}
\end{eqnarray}

To verify the condition, it suffices to show that $x\prec_p y$. Indeed, assumptions~(\ref{sept12-1}) and (\ref{sept12-2}) implies that $x\in \{t,v\}$ and $y\in \{w,u\}$, see Figure~\ref{example gift figure repeated}. Note that $\{t,v\}\prec_p \{w,u\}$, see Figure~\ref{example gift figure repeated}. Therefore, $x\prec_p y$.

\vspace{1mm}
\noindent{\em Condition c:} $t\nVdash \mbox{``Pavel received a gift from Sanaz''}$.
\end{proof}

\begin{proposition}\label{Hs proposition}
$z\Vdash \H_s(\mbox{``Pavel received a gift from Sanaz''})$ iff $z\in\{w\}$.
\end{proposition}
\begin{proof}
$(\Rightarrow):$ Suppose that $z\notin \{w\}$. Thus, either $z\in \{t\}$ or $z\in \{u,v\}$. We consider these two cases separately.

\vspace{1mm}
\noindent{\em Case I}: $z\in \{t\}$. Then, $z\nVdash \mbox{``Pavel received a gift from Sanaz''}$, see Figure~\ref{example gift figure repeated}. Therefore, $z\Vdash \H_s(\mbox{``Pavel received a gift from Sanaz''})$ by item 6(a) of Definition~\ref{sat} and because $z\sim_s z$. 

\vspace{1mm}
\noindent{\em Case II}: $z\in \{u,v\}$. Then, $z\sim_s v$, see Figure~\ref{example gift figure repeated}. Note that, see again Figure~\ref{example gift figure repeated},
$v\nVdash \mbox{``Pavel received a gift from Sanaz''}$. Therefore, by item 6(a) of Definition~\ref{sat}, $z\nVdash \H_s(\mbox{``Pavel received a gift from Sanaz''})$.

\vspace{1mm}
\noindent$(\Leftarrow):$ To show that $w\Vdash \H_s(\mbox{``Pavel received a gift from Sanaz''})$, we verify conditions (a), (b), and (c) of item~\ref{item H} of Definition~\ref{sat}:

\vspace{1mm}
\noindent{\em Condition a:} Consider any world $z$ such that $w\sim_s z$. It suffices to show that $z\Vdash \mbox{``Pavel received a gift from Sanaz''}$. Indeed, assumption $w\sim_s z$ implies that $z=w$, see Figure~\ref{example gift figure repeated}. Note that $w\Vdash \mbox{``Pavel received a gift from Sanaz''}$, see again Figure~\ref{example gift figure repeated}. 

\vspace{1mm}
\noindent{\em Condition b:} The proof is similar to the proof of {\em Condition b} in Proposition~\ref{Hp proposition}.

\vspace{1mm}
\noindent{\em Condition c:} $t\nVdash \mbox{``Pavel received a gift from Sanaz''}$.
\end{proof}

The next proposition shows that Sanaz is happy that Pavel is happy only if she gets the thank you card and, thus, she knows that he received the the gift.

\begin{proposition}\label{HsHp proposition}
$z\Vdash \H_s\H_p(\mbox{``Pavel received a gift from Sanaz''})$ iff $z\in\{w\}$.
\end{proposition}
\begin{proof}
Note that $x\Vdash \mbox{``Pavel received a gift from Sanaz''}$ iff $x\in \{w,u\}$, see Figure~\ref{example gift figure repeated}. Thus, by Proposition~\ref{Hp proposition}, for any world $x\in W$,
\begin{eqnarray*}
&&x\Vdash \mbox{``Pavel received a gift from Sanaz''}\\ 
&&\hspace{5mm}\mbox{ iff } x\Vdash \H_p(\mbox{``Pavel received a gift from Sanaz''}).
\end{eqnarray*}
Hence, by Lemma~\ref{semantic substitution lemma}, for any world $x\in W$,
\begin{eqnarray*}
&&x\Vdash \H_s(\mbox{``Pavel received a gift from Sanaz''})\\ 
&&\hspace{5mm}\mbox{ iff } x\Vdash \H_s\H_p(\mbox{``Pavel received a gift from Sanaz''}).
\end{eqnarray*}
Therefore, $z\Vdash \H_s\H_p(\mbox{``Pavel received a gift from Sanaz''})$ iff $z\in\{w\}$ and Proposition~\ref{Hs proposition}.
\end{proof}

Note that because Sanaz never acknowledges the thank you card, Pavel does not know that Sanaz is happy. Hence, he cannot be happy that she is happy. This is captured in the next proposition. 

\begin{proposition}\label{HpHs proposition}
$z\nVdash \H_p\H_s(\mbox{``Pavel received a gift from Sanaz''})$ for each world $z\in \{w,u,v,t\}$.
\end{proposition}
\begin{proof}
We consider the following two cases separately:

\vspace{1mm}
\noindent{\em Case I}: $z\in \{w,u\}$. Then, $z\sim_p u$, see Figure~\ref{example gift figure repeated}. Note that by Proposition~\ref{Hs proposition},
$$u\nVdash \H_s(\mbox{``Pavel received a gift from Sanaz''}).$$ 
Therefore, $z\nVdash \H_p\H_s(\mbox{``Pavel received a gift from Sanaz''})$ by item 6(a) of Definition~\ref{sat} and the statement  $z\sim_p u$.

\vspace{1mm}
\noindent{\em Case II}: $z\in \{v,t\}$. Then,
$z\nVdash \H_s(\mbox{``Pavel received a gift from Sanaz''})$ by Proposition~\ref{Hs proposition}.
Therefore, $z\nVdash \H_p\H_s(\mbox{``Pavel received a gift from Sanaz''})$ by item 6(a) of Definition~\ref{sat}.
\end{proof}

The proof of the next statement is similar to the proof of Proposition~\ref{HpHs proposition} except that it refers to Proposition~\ref{HsHp proposition} instead of Proposition~\ref{Hs proposition}.
\begin{proposition}
$z\nVdash \H_p\H_s\H_p(\mbox{``Pavel received a gift from Sanaz''})$ for each world $z\in \{w,u,v,t\}$. \qed
\end{proposition}

The next proposition states that Sanaz is sad about Pavel not receiving the gift only if she does not send it. Informally, this proposition is true because Sanaz cannot distinguish world $v$ in which the gift is lost from world $u$ in which the card is lost.

\begin{proposition}\label{Ss proposition}
$z\Vdash \S_s\neg(\mbox{``Pavel received a gift from Sanaz''})$ iff $z\in \{t\}$.
\end{proposition}
\begin{proof}
$(\Rightarrow):$ Suppose that $z\notin \{t\}$. Thus, either $z\in \{w\}$ or $z\in\{u,v\}$. We consider the these two cases separately:

\vspace{1mm}
\noindent{\em Case I}: $z\in \{w\}$. Then, $z\Vdash \mbox{``Pavel received a gift from Sanaz''}$, see Figure~\ref{example gift figure repeated}. Thus, $z\nVdash \neg(\mbox{``Pavel received a gift from Sanaz''})$ by item 2 of Definition~\ref{sat}. Therefore, $z\nVdash \S_s\neg(\mbox{``Pavel received a gift from Sanaz''})$ by item 7(a) of Definition~\ref{sat}.

\vspace{1mm}
\noindent{\em Case II}: $z\in\{u,v\}$. Then, $z\sim_s u$, see Figure~\ref{example gift figure repeated}. Note that  
$$u\Vdash \mbox{``Pavel received a gift from Sanaz''},$$
see Figure~\ref{example gift figure repeated}.
Thus, $u\nVdash \neg(\mbox{``Pavel received a gift from Sanaz''})$ by item 2 of Definition~\ref{sat}. Therefore, $z\nVdash \S_s\neg(\mbox{``Pavel received a gift from Sanaz''})$ by item 7(a) of Definition~\ref{sat} and the statement $z\sim_s u$.

\vspace{1mm}
\noindent$(\Leftarrow):$ To prove that $t\Vdash \S_s\neg(\mbox{``Pavel received a gift from Sanaz''})$, we verify conditions (a), (b), and (c) of item 7 in Definition~\ref{sat} separately:

\vspace{1mm}
\noindent{\em Condition a:} Consider any world $z'$ such that $t\sim_s z'$. It suffices to show that $z'\Vdash \neg(\mbox{``Pavel received a gift from Sanaz''})$. 
Indeed, note that $$t\nVdash (\mbox{``Pavel received a gift from Sanaz''}),$$ see Figure~\ref{example gift figure repeated}. Then, $t\Vdash \neg(\mbox{``Pavel received a gift from Sanaz''})$ by item 2 of Definition~\ref{sat}. Also note that the assumption $t\sim_s z'$ implies that $t=z'$, see Figure~\ref{example gift figure repeated}. Therefore, $z'\Vdash \neg(\mbox{``Pavel received a gift from Sanaz''})$.

\vspace{1mm}
\noindent{\em Condition b:} Consider any two epistemic worlds $x,y$ such that 
\begin{eqnarray*}
    &&x\Vdash \neg(\mbox{``Pavel received a gift from Sanaz''}),\\
    &&y\nVdash \neg(\mbox{``Pavel received a gift from Sanaz''}).
\end{eqnarray*}
To verify the condition, it suffices to show that $x\prec_s y$. Indeed, 
by item 2 of Definition~\ref{sat},
\begin{eqnarray*}
    &&x\nVdash \mbox{``Pavel received a gift from Sanaz''},\\
    &&y\Vdash \mbox{``Pavel received a gift from Sanaz''}.
\end{eqnarray*}
Thus, $x\in \{t,v\}$ and $y\in \{w,u\}$, see Figure~\ref{example gift figure repeated}. Note that $\{t,v\}\prec_s \{w,u\}$, see also Figure~\ref{example gift figure repeated}. Therefore, $x\prec_s y$.

\vspace{1mm}
\noindent{\em Condition c:} Note that $w\Vdash \mbox{``Pavel received a gift from Sanaz''}$. Therefore,
$w\nVdash \neg(\mbox{``Pavel received a gift from Sanaz''})$ by item 2 of Definition~\ref{sat}.
\end{proof}

By Proposition~\ref{Ss proposition}, Sanaz is sad about Pavel not receiving the gift only if she does not send it. Since Pavel cannot distinguish world $t$ in which the gift is sent from world $v$ in which it is lost, Pavel cannot know that Sanaz is sad. This is formally captured in the next proposition.

\begin{proposition}
$z\nVdash \K_p\S_s\neg(\mbox{``Pavel received a gift from Sanaz''})$ for each world $z\in \{w,u,t,v\}$.
\end{proposition}
\begin{proof}
We consider the following two cases separately:

\vspace{1mm}
\noindent{\em Case I:} $z\in\{w,u\}$. Then, $z\nVdash \S_s\neg(\mbox{``Pavel received a gift from Sanaz''})$ by Proposition~\ref{Ss proposition}. Therefore, $z\nVdash \K_p\S_s\neg(\mbox{``Pavel received a gift from Sanaz''})$ by item 5 of Definition~\ref{sat}.

\vspace{1mm}
\noindent{\em Case II:} $z\in\{t,v\}$. Then, $z\sim_p v$, see Figure~\ref{example gift figure repeated}. Note that 
$$v\nVdash \S_s\neg(\mbox{``Pavel received a gift from Sanaz''})$$ by Proposition~\ref{Ss proposition}. Therefore, $z\nVdash \K_p\S_s\neg(\mbox{``Pavel received a gift from Sanaz''})$ by item 5 of Definition~\ref{sat} and the statement $z\sim_p v$.
\end{proof}

\begin{proposition}\label{Sp proposition}
$z\Vdash \S_p\neg(\mbox{``Pavel received a gift from Sanaz''})$ iff $z\in \{v,t\}$.
\end{proposition}
\begin{proof}
$(\Rightarrow):$ Suppose that $z\notin \{v,t\}$. Thus, $z\in\{w,u\}$. Hence, see Figure~\ref{example gift figure repeated}, $z\Vdash \mbox{``Pavel received a gift from Sanaz''}$. Then, by item 2 of Definition~\ref{sat},
$$z\nVdash \neg(\mbox{``Pavel received a gift from Sanaz''}).$$
Therefore, $z\nVdash \S_p\neg(\mbox{``Pavel received a gift from Sanaz''})$ by item 7(a) of Definition~\ref{sat}. 

\vspace{1mm}
\noindent$(\Leftarrow):$ Suppose that $z\in \{v,t\}$. We verify conditions (a), (b), and (c) from item 7 of Definition~\ref{sat} to prove that
$z\Vdash \S_p\neg(\mbox{``Pavel received a gift from Sanaz''})$:

\vspace{1mm}
\noindent{\em Condition a:} Consider any world $z'$ such that $z\sim_p z'$. It suffices to show that $z'\Vdash \neg(\mbox{``Pavel received a gift from Sanaz''})$. Indeed, the assumptions  $z\in \{v,t\}$ and  $z\sim_p z'$ imply that $z'\in \{v,t\}$, see Figure~\ref{example gift figure repeated}. Thus, see again Figure~\ref{example gift figure repeated}, $z'\nVdash \mbox{``Pavel received a gift from Sanaz''}$. Therefore, by item 2 of Definition~\ref{sat}, $z'\Vdash \neg(\mbox{``Pavel received a gift from Sanaz''})$. 

\vspace{1mm}
\noindent{\em Condition b:} The proof is similar to the proof of {\em Condition b} in Proposition~\ref{Ss proposition}.

\vspace{1mm}
\noindent{\em Condition c:} Note that $w\Vdash \mbox{``Pavel received a gift from Sanaz''}$. Therefore,
$w\nVdash \neg(\mbox{``Pavel received a gift from Sanaz''})$ by item 2 of Definition~\ref{sat}.
\end{proof}

\begin{proposition}\label{SsSp proposition}
$z\Vdash \S_s\S_p\neg(\mbox{``Pavel received a gift from Sanaz''})$ iff $z\in \{t\}$.
\end{proposition}
\begin{proof}
Note that $x\nVdash \mbox{``Pavel received a gift from Sanaz''}$ iff $x\in \{t,v\}$, see Figure~\ref{example gift figure repeated}. Thus, $x\Vdash\neg( \mbox{``Pavel received a gift from Sanaz''})$ iff $x\in \{t,v\}$ by item 2 of Definition~\ref{sat}.
Thus, by Proposition~\ref{Sp proposition}, for any world $x\in W$,
\begin{eqnarray*}
&&x\Vdash \neg(\mbox{``Pavel received a gift from Sanaz''})\\ 
&&\hspace{5mm}\mbox{ iff } x\Vdash \S_p\neg(\mbox{``Pavel received a gift from Sanaz''}).
\end{eqnarray*}
Hence, by Lemma~\ref{semantic substitution lemma}, for any world $z\in W$,
\begin{eqnarray*}
&&z\Vdash \S_s\neg(\mbox{``Pavel received a gift from Sanaz''})\\ 
&&\hspace{5mm}\mbox{ iff } z\Vdash \S_s\S_p\neg(\mbox{``Pavel received a gift from Sanaz''}).
\end{eqnarray*}

Therefore, $z\Vdash \S_s\S_p\neg(\mbox{``Pavel received a gift from Sanaz''})$ iff $z\in\{t\}$  by Proposition~\ref{Ss proposition}.
\end{proof}

\begin{proposition}
$z\nVdash \K_p\S_s\S_p\neg(\mbox{``Pavel received a gift from Sanaz''})$ for each world $z\in \{w,u,v,t\}$.
\end{proposition}
\begin{proof}
We consider the following two cases separately:

\vspace{1mm}
\noindent{\em Case I}: $z\in \{w,u\}$. Then, $z\nVdash \S_s\S_p\neg(\mbox{``Pavel received a gift from Sanaz''})$ by Proposition~\ref{SsSp proposition}. Therefore, $z\nVdash \K_p\S_s\S_p\neg(\mbox{``Pavel received a gift from Sanaz''})$ by item 5 of Definition~\ref{sat}.

\noindent{\em Case II}: $z\in \{t,v\}$. Then, $z\sim_p v$, see Figure~\ref{example gift figure repeated}. Note that, by Proposition~\ref{SsSp proposition}, $v\nVdash \S_s\S_p\neg(\mbox{``Pavel received a gift from Sanaz''})$. Therefore, by item 5 of Definition~\ref{sat}, $z\nVdash \K_p\S_s\S_p\neg(\mbox{``Pavel received a gift from Sanaz''})$. 
\end{proof}

\section{The Battle of Cuisines Scenario}\label{The Battle of Cuisines Scenario section}

In this section we illustrate Definition~\ref{sat} using a scenario based on a classical strategic game. 
\begin{table}[ht]
\begin{center}
\begin{tabular}{ ccc }
\toprule
     & Iranian & Russian\\ \midrule
  Iranian  & 1,3 & 0,0 \\ 
  Russian  & 0,0 & 3,1 \\
 \bottomrule
\end{tabular}
\caption{The Battle of Cuisines. Sanaz is the first player, Pavel is the second.}\label{example table figure}
\end{center}
\end{table}
Suppose that the two co-authors finally met and are deciding on a restaurant to have dinner. Sanaz, being Iranian, wants to explore Russian cuisine, while Pavel would prefer to have dinner in an Iranian restaurant. The epistemic worlds in this model are pairs $(r_s,r_p)$ of restaurants choices made by Sanaz and Pavel respectively, where $r_s,r_p\in\{\mbox{Iranian},\mbox{Russian}\}$. We will consider the situation after they both arrived to a restaurant and thus each of them already knows the choice made by the other. Hence, both of them can distinguish all epistemic worlds. In other words, this is a {\em perfect information} scenario. We specify the preference relations of Sanaz and Pavel through their respective utility functions $u_s$ and $u_p$ captured in Table~\ref{example table figure}. For example, $(\mbox{Russian},\mbox{Iranian})\prec_s(\mbox{Iranian},\mbox{Iranian})$ because the value of Sanaz's utility function in world $(\mbox{Iranian},\mbox{Iranian})$  is larger than in world $(\mbox{Russian},\mbox{Iranian})$:
\begin{eqnarray*}
    &&u_s(\mbox{Iranian},\mbox{Iranian})=1,\\
    &&u_s(\mbox{Russian},\mbox{Iranian})=0.
\end{eqnarray*}

\begin{proposition}
$$(\mathrm{Russian},\mathrm{Russian})\nVdash \H_p(\mbox{``Pavel is in the Russian restaurant''}).$$
\end{proposition}
\begin{proof} Note that
\begin{eqnarray*}
    && (\mathrm{Iranian},\mathrm{Iranian})\nVdash \mbox{``Pavel is in the Russian restaurant''},\\
    && (\mathrm{Russian},\mathrm{Russian})\Vdash \mbox{``Pavel is in the Russian restaurant''}.
\end{eqnarray*}
At the same time, see Table~\ref{example table figure},
$$
u_p(\mathrm{Russian},\mathrm{Russian})=1<3=u_p(\mathrm{Iranian},\mathrm{Iranian}).
$$
Hence, $(\mathrm{Iranian},\mathrm{Iranian})\nprec_p(\mathrm{Russian},\mathrm{Russian})$. 
Therefore, by item 6(b) of Definition~\ref{sat}, $(\mathrm{Russian},\mathrm{Russian})\nVdash \H_p(\mbox{``Pavel is in the Russian restaurant''})$.
\end{proof}

Note that Sanaz prefers world $(\mathrm{Russian},\mathrm{Russian})$ to any other world. This, however, does not mean that she is happy about everything in this world. The next two propositions illustrate this.

\begin{proposition}\label{oct19-rr}
$$(\mathrm{Russian},\mathrm{Russian})\nVdash \H_s(\mbox{``Sanaz is in the Russian restaurant''}).$$
\end{proposition}
\begin{proof}
Note that
\begin{eqnarray*}
    && (\mathrm{Iranian},\mathrm{Iranian})\nVdash \mbox{``Sanaz is in the Russian restaurant''},\\
    && (\mathrm{Russian},\mathrm{Iranian})\Vdash \mbox{``Sanaz is in the Russian restaurant''}.
\end{eqnarray*}
At the same time, see Table~\ref{example table figure},
$$
u_s(\mathrm{Russian},\mathrm{Iranian})=0<1=u_s(\mathrm{Iranian},\mathrm{Iranian}).
$$
Hence, $(\mathrm{Iranian},\mathrm{Iranian})\nprec_s(\mathrm{Russian},\mathrm{Iranian})$. 
Therefore, by item 6(b) of Definition~\ref{sat}, $(\mathrm{Russian},\mathrm{Russian})\nVdash \H_s(\mbox{``Sanaz is in the Russian restaurant''})$.
\end{proof}

\begin{proposition}\label{sept25 Hs}
$$(x,y)\Vdash \H_s(\mbox{``Sanaz and Pavel are in the same restaurant''})$$ iff $x=y$.
\end{proposition}
\begin{proof}
$(\Rightarrow):$ Suppose that $x\neq y$. Then,
\begin{eqnarray*}
    && (x,y)\nVdash \mbox{``Sanaz and Pavel are in the same restaurant''},\\
    && (x,x)\Vdash \mbox{``Sanaz and Pavel are in the same restaurant''}.
\end{eqnarray*}
Note also that 
$
u_s(x,y)=0<1\le u_s(x,x)
$
because $x\neq y$, see Table~\ref{example table figure}. Hence
$(x,x)\nprec_s(x,y)$.
Therefore, $$(x,y)\nVdash \H_s(\mbox{``Sanaz and Pavel are in the same restaurant''})$$
by item 6(b) of Definition~\ref{sat}.

\vspace{1mm}
\noindent$(\Leftarrow):$ Suppose $x=y$. We verify conditions (a), (b), and (c) from item 6 of Definition~\ref{sat} to prove that
$(x,y)\Vdash \H_s(\mbox{``Sanaz and Pavel are in the same restaurant''})$:

\vspace{1mm}
\noindent{\em Condition a:} Since this is a model with perfect information, it suffices to show that $(x,y)\Vdash \mbox{``Sanaz and Pavel are in the same restaurant''}$, which is true due to the assumption $x=y$.

\vspace{1mm}
\noindent{\em Condition b:} Consider any two worlds $(x_1,y_1), (x_2,y_2)\in W$ such that 
\begin{eqnarray}
    && (x_1,y_1)\nVdash \mbox{``Sanaz and Pavel are in the same restaurant''},\label{sept25-a}\\
    && (x_2,y_2)\Vdash \mbox{``Sanaz and Pavel are in the same restaurant''}.\label{sept25-b}
\end{eqnarray}
It suffices to show that $(x_1,y_1)\prec_s(x_2,y_2)$. Indeed, statements~(\ref{sept25-a}) and (\ref{sept25-b}) imply that $x_1\neq y_1$ and $x_2=y_2$, respectively. Thus, $u_s(x_1,y_1)=0<1\le u_s(x_2,y_2)$, see Table~\ref{example table figure}. Therefore, $(x_1,y_1)\prec_s(x_2,y_2)$.

\vspace{1mm}
\noindent{\em Condition c:} Note that $$(\mathrm{Russian},\mathrm{Iranian})\nVdash \mbox{``Sanaz and Pavel are in the same restaurant''}.$$
\end{proof}

The proof of the next proposition is similar to the proof of the one above. 
\begin{proposition}\label{sept25 Hp}
$$(x,y)\Vdash \H_p(\mbox{``Sanaz and Pavel are in the same restaurant''})$$ iff $x=y$. \qed
\end{proposition}

\begin{proposition}\label{sept25 HpHs}
$$(x,y)\Vdash \H_p\H_s(\mbox{``Sanaz and Pavel are in the same restaurant''})$$ iff $x=y$.
\end{proposition}
\begin{proof}
Note that $(x,y)\Vdash \mbox{``Sanaz and Pavel are in the same restaurant''}$ iff $x=y$. Thus, by Proposition~\ref{sept25 Hs}, for any world $(x,y)\in W$,
\begin{eqnarray*}
&&(x,y)\Vdash \mbox{``Sanaz and Pavel are in the same restaurant''}\\ 
&&\hspace{5mm}\mbox{ iff } (x,y)\Vdash \H_s(\mbox{``Sanaz and Pavel are in the same restaurant''}).
\end{eqnarray*}
Hence, by Lemma~\ref{semantic substitution lemma}, for any world $(x,y)\in W$,
\begin{eqnarray*}
&&(x,y)\Vdash \H_p(\mbox{``Sanaz and Pavel are in the same restaurant''})\\ 
&&\hspace{5mm}\mbox{ iff } (x,y)\Vdash \H_p\H_s(\mbox{``Sanaz and Pavel are in the same restaurant''}).
\end{eqnarray*}
Therefore, $(x,y)\Vdash \H_p\H_s(\mbox{``Sanaz and Pavel are in the same restaurant''})$ iff $x=y$ by Proposition~\ref{sept25 Hp}.
\end{proof}

\begin{proposition}
$$(x,y)\Vdash \H_s\H_p\H_s(\mbox{``Sanaz and Pavel are in the same restaurant''})$$
iff $x=y$.
\end{proposition}
\begin{proof}
Note that $(x,y)\Vdash \mbox{``Sanaz and Pavel are in the same restaurant''}$ iff $x=y$. Thus, by Proposition~\ref{sept25 HpHs}, for any world $(x,y)\in W$,
\begin{eqnarray*}
&&(x,y)\Vdash \mbox{``Sanaz and Pavel are in the same restaurant''}\\ 
&&\hspace{5mm}\mbox{ iff } (x,y)\Vdash \H_p\H_s(\mbox{``Sanaz and Pavel are in the same restaurant''}).
\end{eqnarray*}
Hence, by Lemma~\ref{semantic substitution lemma}, for any world $(x,y)\in W$,
\begin{eqnarray*}
&&(x,y)\Vdash \H_s(\mbox{``Sanaz and Pavel are in the same restaurant''})\\ 
&&\hspace{5mm}\mbox{ iff } (x,y)\Vdash \H_s\H_p\H_s(\mbox{``Sanaz and Pavel are in the same restaurant''}).
\end{eqnarray*}
Therefore, $(x,y)\Vdash \H_s\H_p\H_s(\mbox{``Sanaz and Pavel are in the same restaurant''})$ iff $x=y$ by Proposition~\ref{sept25 Hs}.
\end{proof}

\begin{proposition}\label{sept26 Hs proposition}
$$(x,y)\Vdash \H_s(\mbox{``Sanaz and Pavel are in the Russian restaurant''})$$
iff $x=y=\mathrm{Russian}$.
\end{proposition}
\begin{proof}
$(\Rightarrow):$ By item 6(a) of Definition~\ref{sat}, the assumption of the proposition $(x,y)\Vdash \H_s(\mbox{``Sanaz and Pavel are in the Russian restaurant''})$ implies that $(x,y)\Vdash \mbox{``Sanaz and Pavel are in the Russian restaurant''}$. Therefore, $x=y=\mathrm{Russian}$.

\vspace{1mm}
\noindent
$(\Leftarrow):$ Suppose $x=y=\mathrm{Russian}$. To prove that
$$(x,y)\Vdash \H_s(\mbox{``Sanaz and Pavel are in the Russian restaurant''}),$$ we verify conditions (a), (b), and (c) from item 6 of Definition~\ref{sat}:

\vspace{1mm}
\noindent{\em Condition a:} Since this is a model with perfect information, it suffices to show that $(x,y)\Vdash \mbox{``Sanaz and Pavel are in the Russian restaurant''}$, which is true due to the assumption $x=y=\mathrm{Russian}$.

\vspace{1mm}
\noindent{\em Condition b:} Consider any two worlds $(x_1,y_1), (x_2,y_2)\in W$ such that 
\begin{eqnarray}
    && (x_1,y_1)\nVdash \mbox{``Sanaz and Pavel are in the Russian restaurant''},\label{sept25-c}\\
    && (x_2,y_2)\Vdash \mbox{``Sanaz and Pavel are in the Russian restaurant''}.\label{sept25-d}
\end{eqnarray}
It suffices to show that $(x_1,y_1)\prec_s(x_2,y_2)$. Indeed, statement~(\ref{sept25-c}) implies that $u_s(x_1,y_1)\le 1$, see Table~\ref{example table figure}. Similarly, statement~(\ref{sept25-d}) implies that $u_s(x_2,y_2)=3$. Thus, $u_s(x_1,y_1)\le 1<3=u_s(x_2,y_2)$. Therefore, $(x_1,y_1)\prec_s(x_2,y_2)$.

\vspace{1mm}
\noindent{\em Condition c:} $(\mathrm{Russian},\mathrm{Iranian})\nVdash \mbox{``Sanaz and Pavel are in the Russian restaurant''}$.
\end{proof}

\begin{proposition}
$$(\mathrm{Russian},\mathrm{Russian})\nVdash \H_p(\mbox{``Sanaz and Pavel are in the Russian restaurant''}).$$
\end{proposition}
\begin{proof}
Note that 
$u_p(\mathrm{Iranian},\mathrm{Iranian})=3>1=u_p(\mathrm{Russian},\mathrm{Russian})$. Thus, $(\mathrm{Iranian},\mathrm{Iranian})\nprec_p(\mathrm{Russian},\mathrm{Russian})$. Therefore, the proposition is true 
by item 6(b) of Definition~\ref{sat}.
\end{proof}

\begin{proposition}
$$(\mathrm{Russian},\mathrm{Russian})\nVdash \H_p\H_s(\mbox{``Sanaz and Pavel are in the Russian restaurant''}).$$
\end{proposition}
\begin{proof}
Note that 
\begin{eqnarray*}
    && (\mathrm{Iranian},\mathrm{Iranian})\nVdash \H_s(\mbox{``Sanaz and Pavel are in the Russian restaurant''}),\label{sept26-c}\\
    &&(\mathrm{Russian},\mathrm{Russian})\Vdash \H_s(\mbox{``Sanaz and Pavel are in the Russian restaurant''}).\label{sept26-d}
\end{eqnarray*}
by Proposition~\ref{sept26 Hs proposition}. At the same time, $$u_p(\mathrm{Iranian},\mathrm{Iranian})=3>1=u_p(\mathrm{Russian},\mathrm{Russian}).$$ 
Thus, $(\mathrm{Iranian},\mathrm{Iranian})\nprec_p(\mathrm{Russian},\mathrm{Russian})$. Therefore, the proposition is true 
by item 6(b) of Definition~\ref{sat}.
\end{proof}

\begin{proposition}\label{sept26-Ss}
$$(x,y)\Vdash \S_s(\mbox{``Sanaz and Pavel are in different restaurants''})\;\;\mbox{ iff }\;\; x\neq y.$$
\end{proposition}
\begin{proof}
$(\Rightarrow):$ Suppose that $x=y$. Thus,
$$(x,y)\nVdash \mbox{``Sanaz and Pavel are in different restaurants''}.$$
Thus, by item 7(a) of Definition~\ref{sat},
$$(x,y)\nVdash \S_s(\mbox{``Sanaz and Pavel are in different restaurants''}).$$

\vspace{1mm}
\noindent$(\Leftarrow):$ Suppose that $x\neq y$.  To prove that
$$(x,y)\Vdash \S_s(\mbox{``Sanaz and Pavel are in different restaurants''}),$$ we verify conditions (a), (b), and (c) from item 7 of Definition~\ref{sat}:

\vspace{1mm}
\noindent{\em Condition a:} Since this is a model with perfect information, it suffices to show that $(x,y)\Vdash \mbox{``Sanaz and Pavel are in different restaurants''}$, which is true due to the assumption $x\neq y$.

\vspace{1mm}
\noindent{\em Condition b:} The proof of this condition is similar to the proof of {\em Condition b} in the proof of Proposition~\ref{sept25 Hs}.

\vspace{1mm}
\noindent{\em Condition c:} $(\mathrm{Iranian},\mathrm{Iranian})\nVdash \mbox{``Sanaz and Pavel are in different restaurants''}$.
\end{proof}


The proof of the following statement is similar to the proof of the one above.
\begin{proposition}\label{sept26-Sp}
$$(x,y)\Vdash \S_p(\mbox{``Sanaz and Pavel are in different restaurants''}) \;\;\mbox{ iff }\;\;x\neq y.$$ \qed
\end{proposition}

\begin{proposition}
$$(x,y)\Vdash \S_p\S_s(\mbox{``Sanaz and Pavel are in different restaurants''})\;\;\mbox{ iff }\;\; x\neq y.$$
\end{proposition}
\begin{proof}
Note that $(x,y)\Vdash \mbox{``Sanaz and Pavel are in different restaurants''}$ iff $x\neq y$. Thus, by Proposition~\ref{sept26-Ss}, for any world $(x,y)\in W$,
\begin{eqnarray*}
&&(x,y)\Vdash \mbox{``Sanaz and Pavel are in different restaurants''}\\ 
&&\hspace{5mm}\mbox{ iff } (x,y)\Vdash \S_s(\mbox{``Sanaz and Pavel are in different restaurants''}).
\end{eqnarray*}
Hence, by Lemma~\ref{semantic substitution lemma}, for any world $(x,y)\in W$,
\begin{eqnarray*}
&&(x,y)\Vdash \S_p(\mbox{``Sanaz and Pavel are in different restaurants''})\\ 
&&\hspace{5mm}\mbox{ iff } (x,y)\Vdash \S_p\S_s(\mbox{``Sanaz and Pavel are in different restaurants''}).
\end{eqnarray*}
Therefore, $(x,y)\Vdash \S_p\S_s(\mbox{``Sanaz and Pavel are in different restaurants''})$ iff $x=y$ by Proposition~\ref{sept26-Sp}.
\end{proof}

\section{Lottery Example}

\epigraph{Be happy for this moment. This moment is your life.}%
{The Rub\'{a}iy\'{a}t of Omar Khayy\'{a}m}


As our next example, consider a hypothetical situation when Sanaz and  Pavel play lottery with Omar Khayy\'{a}m, an Iranian mathematician, astronomer, philosopher, and poet. Each of them gets a lottery ticket and it is known that exactly one out of three tickets is winning. We consider the moment when each of the players has already seen her or his own ticket, but does not know yet what are the tickets of the other players. We assume that each of the three players prefers the outcome when she or he wins the lottery.

\begin{figure}[ht]
\begin{center}
\scalebox{0.45}{\includegraphics{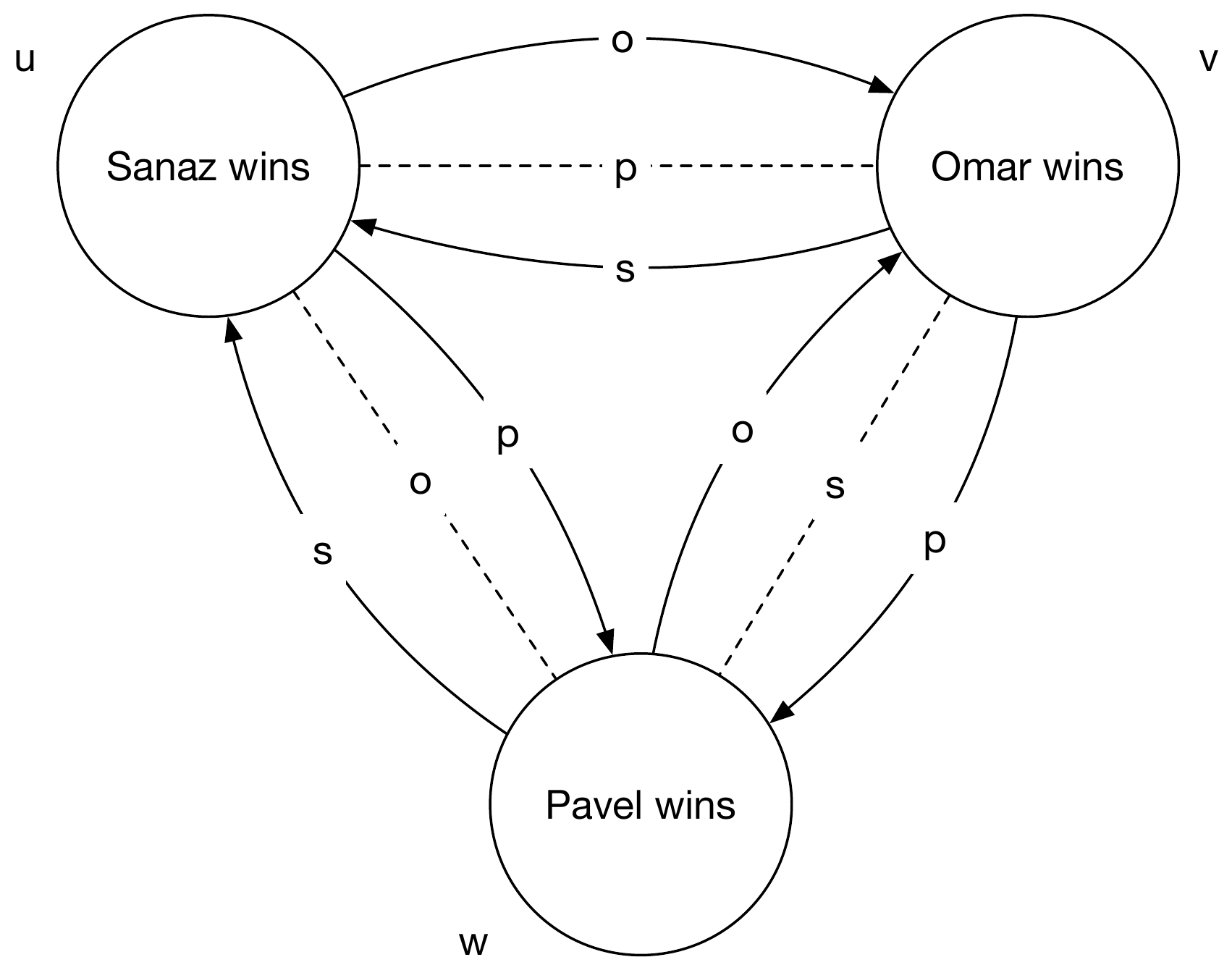}}
\caption{Lottery Epistemic Model with Preferences.}\label{example lottery figure}
\end{center}
\end{figure}

Figure~\ref{example lottery figure} depicts the epistemic model with preferences capturing the above scenario. It has three epistemic worlds, $u$, $v$, and $w$ in which the winner is Sanaz, Omar, and Pavel respectively. Dashed lines represent indistinguishability relations. For example, the dashed line between worlds $w$ and $v$ labeled with $s$ shows that Sanaz cannot distinguish the world in which Pavel wins from the one in which Omar wins. This is true because we consider the knowledge at the moment when neither of the players knows yet what are the tickets of the other players. The directed edges between worlds represent preference relations. For example, the directed edge from world $w$ to world $v$ labeled with $o$ captures the fact that Omar would prefer to win the lottery rather than to lose it.

\begin{proposition}\label{oct3-Hs}
$x\Vdash \H_s(\mbox{``Sanaz won the lottery''})$ iff $x=u$.
\end{proposition}
\begin{proof}
$(\Rightarrow):$ Suppose that $x\neq u$. Thus,
$x\nVdash \mbox{``Sanaz won the lottery''}$, see Figure~\ref{example lottery figure}. Therefore, $x\nVdash \H_s(\mbox{``Sanaz won the lottery''})$ by item 6(a) of Definition~\ref{sat}. 

\vspace{1mm}
\noindent$(\Leftarrow):$ Suppose that $x=u$.  To prove the required, it suffices to verify conditions (a), (b), and (c) from item 6 of Definition~\ref{sat}:

\vspace{1mm}
\noindent{\em Condition a:} Consider any world $y$ such that $u\sim_s y$. We will show that $y\Vdash \mbox{``Sanaz won the lottery''}$. Indeed, assumption $u\sim_s y$ implies that $u=y$, see Figure~\ref{example lottery figure}. Therefore, $y\Vdash \mbox{``Sanaz won the lottery''}$, see again Figure~\ref{example lottery figure}.

\vspace{1mm}
\noindent{\em Condition b:} Consider any worlds $y,z$ such that $y\nVdash \mbox{``Sanaz won the lottery''}$ and $z\Vdash \mbox{``Sanaz won the lottery''}$. We will show that $y\prec_s z$. Indeed, assumption $y\nVdash \mbox{``Sanaz won the lottery''}$ implies that $y\in \{w,v\}$, see Figure~\ref{example lottery figure}. Similarly, assumption $z\Vdash \mbox{``Sanaz won the lottery''}$ implies that $z=u$. Statements $y\in \{w,v\}$ and $z=u$ imply that $y\prec_s z$, see again Figure~\ref{example lottery figure}.

\vspace{1mm}
\noindent{\em Condition c:} $w\nVdash \mbox{``Sanaz won the lottery''}$.
\end{proof}

\begin{proposition}
$u\nVdash \H_s(\mbox{``Pavel lost the lottery''})$.
\end{proposition}
\begin{proof}
Note that $w\nVdash \mbox{``Pavel lost the lottery''}$, $v\Vdash \mbox{``Pavel lost the lottery''}$, and $w\nprec_s v$, see  Figure~\ref{example lottery figure}. Therefore, $u\nVdash \H_s(\mbox{``Pavel lost the lottery''})$ by item 6(b) of Definition~\ref{sat}. 
\end{proof}

\begin{proposition}
$u\nVdash \K_p\H_s(\mbox{``Sanaz won the lottery''})$.
\end{proposition}
\begin{proof}
Note that $u\sim_p v$, see  Figure~\ref{example lottery figure}. Also,
$v\nVdash \H_s(\mbox{``Sanaz won the lottery''})$ by Proposition~\ref{oct3-Hs}. Therefore, $u\nVdash \K_p\H_s(\mbox{``Sanaz won the lottery''})$ by item 5 of Definition~\ref{sat}.
\end{proof}

\begin{proposition}\label{oct3-Sp}
$u\Vdash \S_p(\mbox{``Pavel lost the lottery''})$.
\end{proposition}
\begin{proof}
It suffices to verify conditions (a), (b), and (c) from item 7 of Definition~\ref{sat}:

\vspace{1mm}
\noindent{\em Condition a:} Consider any world $y$ such that $u\sim_p y$. We will show that $y\Vdash \mbox{``Pavel lost the lottery''}$. Indeed, assumption $u\sim_p y$ implies that $y\in\{u,v\}$, see Figure~\ref{example lottery figure}. Therefore, $y\Vdash \mbox{``Pavel lost the lottery''}$, see again Figure~\ref{example lottery figure}.

\vspace{1mm}
\noindent{\em Condition b:} Consider any worlds $y,z$ such that $y\Vdash \mbox{``Pavel lost the lottery''}$ and $z\nVdash \mbox{``Pavel lost the lottery''}$. We will show that $y\prec_p z$. Indeed, assumption $y\Vdash \mbox{``Pavel lost the lottery''}$ implies that $y\in \{u,v\}$, see Figure~\ref{example lottery figure}. Similarly, assumption $z\nVdash \mbox{``Pavel lost the lottery''}$ implies that $z=w$. Statements $y\in \{u,v\}$ and $z=w$ imply that $y\prec_s z$, see again Figure~\ref{example lottery figure}.

\vspace{1mm}
\noindent{\em Condition c:} $w\nVdash \mbox{``Pavel lost the lottery''}$.
\end{proof}

\begin{proposition}
$u\Vdash \K_s\S_p(\mbox{``Pavel lost the lottery''})$.
\end{proposition}
\begin{proof}
Consider any world $y$ such that $u\sim_s y$. By item 5 of Definition~\ref{sat}, it suffices to show that $y\Vdash \S_p(\mbox{``Pavel lost the lottery''})$. Indeed, assumption $u\sim_s y$ implies that $u=y$, see Figure~\ref{example lottery figure}. Therefore, $y\Vdash \S_p(\mbox{``Pavel lost the lottery''})$ by Proposition~\ref{oct3-Sp}. 
\end{proof}

\section{Undefinability of Sadness through Happiness}\label{Undefinability of Sadness through Happiness section}

In this section we prove that sadness is not definable through happiness. More formally, we show that formula $\S_a p$ is not equivalent to any formula in the language $\Phi^{-\S}$:
$$
\phi:=p\;|\;\neg\phi\;|\;\phi\to\phi\;|\;\N\phi\;|\;\K_a\phi\;|\;\H_a\phi,
$$
which is obtained by removing modality $\S$ from the full language $\Phi$ of our logical system. In the next section we will use a duality principle to claim that happiness is not definable through sadness either.

Without loss of generality, in this section we assume that the set of agents $\mathcal{A}$ contains a single agent $a$ and the set of propositional variables contains a single variable $p$. To prove undefinability of sadness through happiness we consider two epistemic models with preferences depicted in Figure~\ref{two models figure}. By $\Vdash_l$ and $\Vdash_r$ we mean the satisfaction relations for the left and and the right model respectively.

In Lemma~\ref{nov30-a}, we show that these two models are indistinguishable in language the $\Phi^{-\S}$. In Lemma~\ref{nov30-b} and Lemma~\ref{nov30-c}, we prove that $w_1\Vdash_l \S_a p$ and $w_1\nVdash_r \S_a p$ respectively. Together, these three statements imply undefinability of modality $\S$ in the language $\Phi^{-\S}$, which is stated in the end of this section as Theorem~\ref{dec1-S-is-not-definable}. We start with two auxiliary lemmas used in the proof of Lemma~\ref{nov30-a}.

\begin{figure}[ht]
\begin{center}
\scalebox{0.45}{\includegraphics{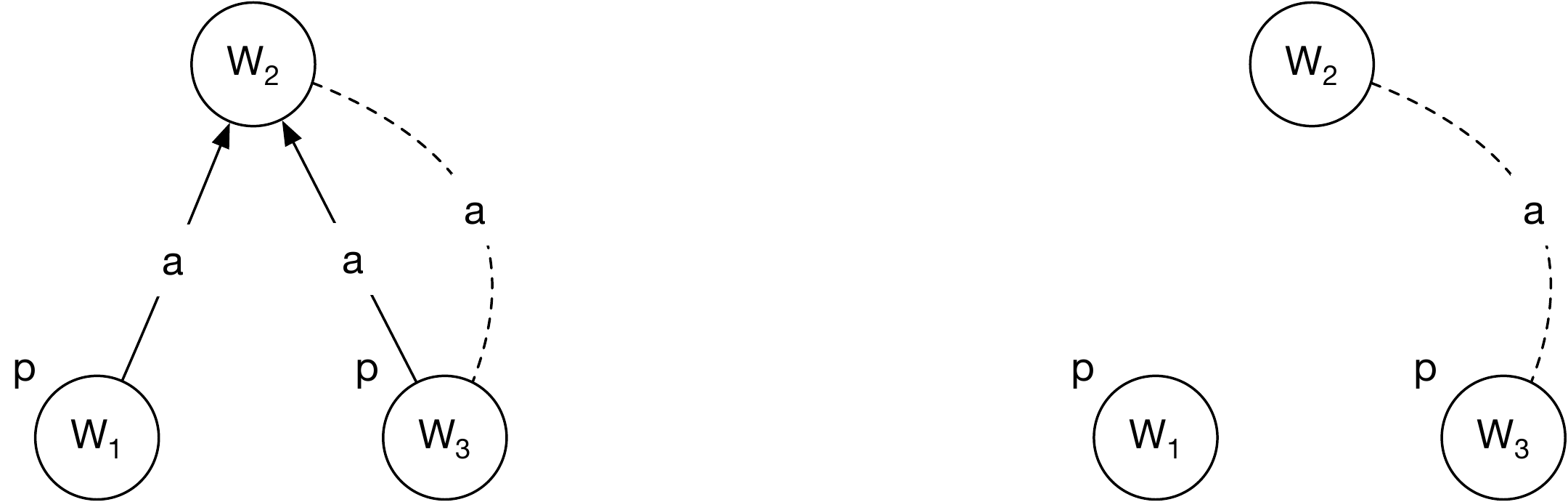}}
\caption{Two Models.}\label{two models figure}
\end{center}
\end{figure}

\begin{lemma}\label{H l lemma}
$w\nVdash_l \H_a \phi$ for each world $w\in\{w_1,w_2,w_3\}$ and each formula $\phi\in \Phi$.
\end{lemma}
\begin{proof}
Suppose that $w\Vdash_l \H_a \phi$. Thus, by item 6 of Definition~\ref{sat}, 
\begin{eqnarray}
&&u\Vdash_l\phi \mbox{ for each world $u\in W$ such that } w\sim_a u,\label{item a}\\
&&u\prec_a u' \mbox{ for all worlds $u,u'\in W$ such that }
    u\nVdash\phi  \mbox{ and } u'\Vdash \phi,\label{item b}
\end{eqnarray}
and there is a world $\widehat{w}\in\{w_1,w_2,w_3\}$ such that
\begin{equation}\label{item c}
    \widehat{w}\nVdash\phi.
\end{equation}
Since relation $\sim_a$ is reflexive, statement~(\ref{item a}) implies that 
\begin{equation}\label{item 1}
    w\Vdash\phi.
\end{equation}
Thus, using statements~(\ref{item b}) and (\ref{item c}),
\begin{equation}\label{item 2}
    \widehat{w}\prec_a w. 
\end{equation}
Hence, see Figure~\ref{two models figure} (left),
\begin{eqnarray}
    &&\widehat{w}\in\{w_1,w_3\},\label{item 3a}\\
    &&w=w_2. \label{item 3b}
\end{eqnarray}
Note that $w_2\sim_a w_3$, see Figure~\ref{two models figure} (left). Thus, $w\sim_a w_3$ by statement~(\ref{item 3b}). Then, by statement~(\ref{item a}),
\begin{equation}\label{item 4}
    w_3\Vdash\phi.
\end{equation}
Hence, $\widehat{w}\neq w_3$ because of statement~(\ref{item c}). Thus, 
$\widehat{w}=w_1$
due to statement~(\ref{item 3a}).
Then,
$w_1\nVdash\phi$ because of statement~(\ref{item c}).
Therefore,
$w_1\prec_a w_3$ by statements~(\ref{item b}) and (\ref{item 4}), which is a contradiction, see Figure~\ref{two models figure} (left).
\end{proof}

\begin{lemma}\label{H r lemma}
$w\nVdash_r \H_a \phi$ for each world $w\in\{w_1,w_2,w_3\}$ and each formula $\phi\in \Phi$.
\end{lemma}
\begin{proof}
Suppose $w\Vdash_r \H_a \phi$. Thus, $w\Vdash_r \phi$ by item 6(a) of Definition~\ref{sat} and the reflexivity of relation $\sim_a$. At the same time, by item 6(c) of Definition~\ref{sat}, there must exist a world $u\in W$ such that $u\nVdash_r\phi$. By item 6(b) of the same Definition~\ref{sat}, the assumption $w\Vdash_r \H_a \phi$ and the statements $u\nVdash_r\phi$ and  $w\Vdash_r \phi$ imply that $u\prec_a w$, which is a contradiction because relation $\prec_a$ in the right model is empty, see Figure~\ref{two models figure}.
\end{proof}

\begin{lemma}\label{nov30-a}
$w\Vdash_l \phi$ iff $w\Vdash_r \phi$ for each world $w$ and each formula $\phi\in \Phi^{-\S}$.
\end{lemma}
\begin{proof}
We prove the statement by structural induction on formula $\phi$. For the case when formula $\phi$ is propositional variable $p$, observe that $\pi_l(p)=\{w_1,w_3\}=\pi_r(p)$ by the choice of the models, see Figure~\ref{two models figure}. Thus, $w\Vdash_l p$ iff $w\Vdash_r p$ for any world $w$ by item 1 of Definition~\ref{sat}. The case when formula $\phi$ is a negation or an implication follows from the induction hypothesis and items 2 and 3 of Definition~\ref{sat} in the straightforward way.

Suppose formula $\phi$ has the form $\N\psi$. If $w\nVdash_l\N\psi$, then by item 4 of Definition~\ref{sat}, there must exists a world $u\in \{w_1,w_2,w_3\}$ such that $u\nVdash_l\psi$. Hence, by the induction hypothesis, $u\nVdash_r\psi$. Therefore, $w\nVdash_r\N\psi$, by item 4 of Definition~\ref{sat}. The proof in the other direction is similar.

Assume that formula $\phi$ has the form $\K_a\psi$. If $w\nVdash_l\K_a\psi$, then, by item 5 of Definition~\ref{sat}, there is a world $u\in \{w_1,w_2,w_3\}$ such that $w\sim^l_a u$ and $u\nVdash_l \psi$. Then, $w\sim^r_a u$ because relations $\sim^l$ and $\sim^r$ are equal, see Figure~\ref{two models figure} and, by the induction hypothesis, $u\nVdash_r \psi$. Therefore, $w\nVdash_r\K_a\psi$ by item 5 of Definition~\ref{sat}. The proof in the other direction is similar.

Finally, suppose formula $\phi$ has the form $\H_a\psi$. Therefore, $w\nVdash_l\phi$ and $w\nVdash_r\phi$ by Lemma~\ref{H l lemma} and Lemma~\ref{H r lemma} respectively.
\end{proof}

\begin{lemma}\label{nov30-b}
$w_1\Vdash_l \S_a p$.
\end{lemma}
\begin{proof}
We verify the three conditions from item 7 of Definition~\ref{sat} separately:
\vspace{1mm}
\noindent{\em Condition a:} Observe that $w_1\in \{w_1,w_3\}=\pi_l(p)$, see Figure~\ref{two models figure} (left). Then, $w_1\Vdash_l p$ by item 1 of Definition~\ref{sat}. Note also that there is only one world $u\in\{w_1,w_2,w_3\}$ such that $w_1\sim^l_a u$ (namely, the world $w_1$ itself), see Figure~\ref{two models figure} (left). Therefore, $u\Vdash p$ for each world $u\in \{w_1,w_2,w_3\}$ such that $w_1\sim_a u$.
    
\vspace{1mm}
\noindent{\em Condition b:} Note that $\pi_l(p)=\{w_1,w_3\}$, see Figure~\ref{two models figure} (left). Then,
    $w_1\Vdash_l p$, 
    $w_3\Vdash_l p$, and
    $w_2\nVdash_l p$ by item 1 of Definition~\ref{sat}. Also, observe that $w_1\prec^l_a w_2$ and  $w_3\prec^l_a w_2$, see Figure~\ref{two models figure} (left). Thus, for any world $u,u'\in \{w_1,w_2,w_3\}$, if $u\Vdash p$ and $u'\nVdash p$, then $u\prec_a u'$.
    
\vspace{1mm}
\noindent{\em Condition c:} $w_2\nVdash_l p$ by item 1 of Definition~\ref{sat} and because $w_2\notin \{w_1,w_3\}=\pi_l(p)$, see Figure~\ref{two models figure} (left).

\vspace{1mm}
This concludes the proof of the lemma.
\end{proof}

\begin{lemma}\label{nov30-c}
$w_1\nVdash_r \S_a p$.
\end{lemma}
\begin{proof}
Note that $\pi_r(p)=\{w_1,w_3\}$, see Figure~\ref{two models figure} (right). Thus, $w_1\Vdash_r p$ and $w_2\nVdash_r p$ by item 1 of Definition~\ref{sat}. Observe also, that $w_1\nprec_a w_2$, Figure~\ref{two models figure} (right). Therefore, $w_1\nVdash_r \S_a p$ by item 7(b) of Definition~\ref{sat}.
\end{proof}

The next theorem follows from the three lemmas above.

\begin{theorem}\label{dec1-S-is-not-definable}
Modality $\S$ is not definable in language $\Phi^{-\S}$. \qed
\end{theorem}

\section{Duality of Happiness and Sadness}\label{Duality of Happiness and Sadness}

As we have shown in the previous section, sadness modality is not definable through happiness modality. In spite of this, there still is a connection between these two modalities captured below in Theorem~\ref{duality theorem}. To understand this connection, we need to introduce the notion of a converse model and the notion of $\tau$-translation. As usual, for any binary relation $R\subseteq X\times Y$, by converse relation $R^{\c}$ we mean the set of pairs $\{(y,x)\in Y\times X\;|\; (x,y)\in R\}$.

\begin{definition}\label{converse model}
By the converse model $M^{\c}$ of an epistemic model with preference $M=(W,\{\sim_a\}_{a\in\mathcal{A}},\{\prec_a\}_{a\in\mathcal{A}},\pi)$, we mean model $(W,\{\sim_a\}_{a\in\mathcal{A}},\{\prec^{\c}_a\}_{a\in\mathcal{A}},\pi)$.
\end{definition}
For any epistemic model with preference, by $\Vdash$ we denote the satisfaction relation for this model and by $\Vdash^\c$  the satisfaction relation for the converse model.

\begin{definition}\label{tau translation}
For any formula $\phi\in\Phi$, formula $\tau(\phi)\in \Phi$ is defined recursively as follows:
\begin{enumerate}
    \item $\tau(p) \equiv p$, where $p$ is a propositional variable, 
    \item $\tau(\neg\phi) \equiv \neg\tau(\phi)$,
    \item $\tau(\phi\to\psi) \equiv \tau(\phi)\to\tau(\psi)$,
    \item $\tau(\N\phi) \equiv \N\tau(\phi)$,
    \item $\tau(\K_a\phi) \equiv \K_a\tau(\phi)$,
    \item $\tau(\H_a\phi) \equiv \S_a\tau(\phi)$,
    \item $\tau(\S_a\phi) \equiv \H_a\tau(\phi)$.
\end{enumerate}
\end{definition}

We are now ready to state and prove the ``duality principle'' that connects modalities $\H$ and $\S$. 

\begin{theorem}\label{duality theorem}
$w\Vdash \phi$ iff $w\Vdash^\c \tau(\phi)$, for each world $w$ of an epistemic model with preferences.
\end{theorem}
\begin{proof}
We prove the theorem by induction on structural complexity of formula $\phi$. If $\phi$ is a propositional variable, then the statement of the theorem holds because the model and the converse model have the same valuation function $\pi$. If $\phi$ is a negation, an implication, an $\N$-formula, or an $\K$-formula, then the statement of the theorem follows from the induction hypothesis and items 2, 3, 4, and 5 of Definition~\ref{sat} respectively.

Suppose that formula $\phi$ has the form $\H_a\psi$. First, assume that $w\Vdash\H_a\psi$. Thus, by item 6 of Definition~\ref{sat}, the following three conditions are satisfied:
        \begin{enumerate}
            \item[(a)] $u\Vdash \psi$ for each world $u\in W$ such that $w\sim_a u$,
            \item[(b)] for any two worlds $u,u'\in W$, if $u\nVdash\psi$ and $u'\Vdash\psi$, then $u\prec_a u'$,
            \item[(c)] there is a world $u\in W$ such that $u\nVdash\psi$.
        \end{enumerate}
Hence, by the induction hypothesis,
\begin{enumerate}
            \item[(d)] $u\Vdash^\c \tau(\psi)$ for each world $u\in W$ such that $w\sim_a u$,
            \item[(e)] for any two worlds $u,u'\in W$, if $u\nVdash^\c\tau(\psi)$ and $u'\Vdash^\c\tau(\psi)$, then $u\prec_a u'$,
            \item[(f)] there is a world $u\in W$ such that $u\nVdash^\c\tau(\psi)$.
        \end{enumerate}
Note that statement (e) is logically equivalent to
\begin{enumerate}
            \item[(g)] for any two worlds $u,u'\in W$, if $u\Vdash^\c\tau(\psi)$ and $u'\nVdash^\c\tau(\psi)$, then $u'\prec_a u$.
        \end{enumerate}
By the definition of converse partial order, statement (g) is equivalent to
\begin{enumerate}
            \item[(h)] for any two worlds $u,u'\in W$, if $u\Vdash^\c\tau(\psi)$ and $u'\nVdash^\c\tau(\psi)$, then $u\prec^\c_a u'$.
        \end{enumerate}
Thus, $w\Vdash^\c\S_a\tau(\psi)$ by item 7 of Definition~\ref{sat} using statements (d), (h), and (f). Therefore, $w\Vdash^\c\tau(\H_a\psi)$. The proof in the other direction and the proof for the case when formula $\phi$ has the form $\S_a\psi$ are similar. 
\end{proof}

The next theorem follows from Theorem~\ref{dec1-S-is-not-definable} and Theorem~\ref{duality theorem}.

\begin{theorem}
Modality $\H$ is not definable in language $\Phi^{-\H}$. \qed
\end{theorem}

\section{Axioms of Emotions}\label{Axioms of Emotions}

In addition to propositional tautologies in language $\Phi$, our logical system contains the following axioms, where $\E\in\{\H,\S\}$. 

\begin{enumerate}
    \item Truth: $\N\phi\to\phi$, $\K_a\phi\to\phi$, and $\E_a\phi\to\phi$,
    \item Distributivity:\\ 
    $\N(\phi\to\psi)\to(\N\phi\to\N\psi)$,\\
    $\K_a(\phi\to\psi)\to(\K_a\phi\to\K_a\psi)$,
    \item Negative Introspection: 
    $\neg\N\phi\to\N\neg\N\phi$, and 
    $\neg\K_a\phi\to\K_a\neg\K_a\phi$,
    \item Knowledge of Necessity: $\N\phi\to\K_a\phi$,
    \item Emotional Introspection: $\E_a\phi\to\K_a\E_a\phi$,
    \item Emotional Consistency: $\H_a\phi\to\neg\S_a\phi$,
    \item Coherence of Potential Emotions:\\
    $\cN\E_a\phi\wedge\cN\E_a\psi\to \N(\phi\to\psi)\vee \N(\psi\to\phi)$,\\
    $\cN\H_a\phi\wedge\cN\S_a\psi\to \N(\phi\to\neg\psi)\vee \N(\neg\psi\to\phi)$,
    \item Counterfactual: $\E_a\phi\to\neg\N\phi$,
    \item Emotional Predictability:\\ $\cN\H_a\phi\vee\cN\S_a\neg\phi\to(\K_a\phi\to\H_a\phi)$,\\
    $\cN\H_a\neg\phi\vee\cN\S_a\phi\to(\K_a\phi\to\S_a\phi)$,
    \item Substitution: $\N(\phi\leftrightarrow\psi)\to (\E_a\phi\to\E_a\psi)$.
\end{enumerate}
The Truth, the Distributivity, and the Negative Introspection axioms for modalities $\N$ and $\K$ are well-known properties from modal logic S5. The Truth axiom for modality $\E$ states that if an agent is either happy or sad about $\phi$, then statement $\phi$ must be true. This axiom reflects the fact that our emotions are defined through agent's knowledge. We will mention belief-based emotions in the conclusion.

The Knowledge of Necessity axioms states that each agent knows all statements that are true in all worlds of the model. The Emotional Introspection axiom captures one of the two possible interpretations of the title of this article: each agent knows her emotions. The other interpretation of the title is stated below as Lemma~\ref{E to K}.
The Emotional Consistency axiom states that an agent cannot be simultaneously happy and sad about the same thing.

Recall that $\cN\phi$ stands for formula $\neg\N\neg\phi$. Thus, formula $\cN\E_a\phi$ means that agent $a$ {\em might} have the emotion $\E$ about statement $\phi$. In other words, formula $\cN\E_a\phi$ states that agent $a$ has a {\em potential} emotion $\E$ about $\phi$. The Coherence of Potential Emotions axioms expresses the fact that potential emotions of any agent are not independent. The first of these axioms says that if an agent has {\em the same} potential emotions about statements $\phi$ and $\psi$, then one of these statements must imply the other in each world of our model. The second of these axioms says that if an agent has {\em opposite} potential emotions about statements $\phi$ and $\psi$, then these statements cannot be both true in any world of the model. We prove soundness of these axioms along with the soundness of the other axioms of our system in Section~\ref{Soundness section}.

The Counterfactual axiom states that an agent cannot have an emotion about something which is universally true in the model. This axiom reflects items 6(c) and 7(c) of Definition~\ref{sat}.

Because the assumptions of both Emotional Predictability axioms contain disjunctions, each of these axioms could be split into two statements. The first statement of the first Emotional Predictability axiom says that if an agent is potentially happy about $\phi$, then she must be happy about $\phi$ each time when she knows that $\phi$ is true. The second statement of the same axiom says that if an agent is potentially sad about $\neg\phi$, then she must be happy about $\phi$ each time when she knows that $\phi$ is true. The second Emotional Predictability axioms is the dual form of the first axiom.

Finally, the Substitution axiom states that if two statements are equivalent in each world of the model and an agent has an emotion about one of them, then she must have the same emotion about the other statement.

We write $\vdash\phi$ and say that statement $\phi$ is a {\em theorem} of our logical system if $\phi$ is derivable from the above axioms using the Modus Ponens and the Necessitation inference rules:
$$
\dfrac{\phi, \phi\to\psi}{\psi}\hspace{20mm}
\dfrac{\phi}{\N\phi}.
$$
For any set of statements $X\subseteq\Phi$, we write $X\vdash\phi$ if formula $\phi$ is derivable from the theorems of our system and the set $X$ using only the Modus Ponens inference rule. We say that set $X$ is {\em inconsistent} if there is a formula $\phi\in\Phi$ such that $X\vdash\phi$ and $X\vdash\neg\phi$.

\begin{lemma}\label{K Necessitation}
Inference rule $\dfrac{\phi}{\K_a\phi}$ is derivable.
\end{lemma}
\begin{proof}
This rule is a combination of the Necessitation inference rule, the Knowledge of Necessity axiom, and the Modus Ponens inference rule.
\end{proof}



\begin{lemma}\label{E to K}
$\vdash \E_a\phi\to\K_a\phi$.
\end{lemma}
\begin{proof}
Note that $\vdash \E_a\phi\to\phi$ by the Truth axiom. Thus,
$\vdash \K_a(\E_a\phi\to\phi)$ by Lemma~\ref{K Necessitation}.
Hence, 
$\vdash \K_a\E_a\phi\to\K_a\phi$
by the Distributivity axiom and the Modus Ponens inference rule.
Therefore, 
$\vdash \E_a\phi\to\K_a\phi$ by the Emotional Introspection axiom and propositional reasoning.
\end{proof}

The next three lemmas are well-known in model logic. We reproduce their proofs here to keep the article self-contained.

\begin{lemma}[deduction]\label{deduction lemma}
If $X,\phi\vdash\psi$, then $X\vdash\phi\to\psi$.
\end{lemma}
\begin{proof}
Suppose that sequence $\psi_1,\dots,\psi_n$ is a proof from set $X\cup\{\phi\}$ and the theorems of our logical system that uses the Modus Ponens inference rule only. In other words, for each $k\le n$, either
\begin{enumerate}
    \item $\vdash\psi_k$, or
    \item $\psi_k\in X$, or
    \item $\psi_k$ is equal to $\phi$, or
    \item there are $i,j<k$ such that formula $\psi_j$ is equal to $\psi_i\to\psi_k$.
\end{enumerate}
It suffices to show that $X\vdash\phi\to\psi_k$ for each $k\le n$. We prove this by induction on $k$ through considering the four cases above separately.

\vspace{1mm}
\noindent{\em Case I}: $\vdash\psi_k$. Note that $\psi_k\to(\phi\to\psi_k)$ is a propositional tautology, and thus, is an axiom of our logical system. Hence, $\vdash\phi\to\psi_k$ by the Modus Ponens inference rule. Therefore, $X\vdash\phi\to\psi_k$. 

\vspace{1mm}
\noindent{\em Case II}: $\psi_k\in X$. Then, $X\vdash\psi_k$.

\vspace{1mm}
\noindent{\em Case III}: formula $\psi_k$ is equal to $\phi$. Thus, $\phi\to\psi_k$ is a propositional tautology. Therefore, $X\vdash\phi\to\psi_k$. 

\vspace{1mm}
\noindent{\em Case IV}:  formula $\psi_j$ is equal to $\psi_i\to\psi_k$ for some $i,j<k$. Thus, by the induction hypothesis, $X\vdash\phi\to\psi_i$ and $X\vdash\phi\to(\psi_i\to\psi_k)$. Note that formula 
$
(\phi\to\psi_i)\to((\phi\to(\psi_i\to\psi_k))\to(\phi\to\psi_k))
$
is a propositional tautology. Therefore, $X\vdash \phi\to\psi_k$ by applying the Modus Ponens inference rule twice.
\end{proof}

\begin{lemma}\label{superdistributivity lemma} 
If $\phi_1,\dots,\phi_n\vdash\psi$, then $\Box\phi_1,\dots,\Box\phi_n\vdash\Box\psi$, where $\Box$ is either modality $\N$ or modality $\K_a$. 
\end{lemma}
\begin{proof}
Lemma~\ref{deduction lemma} applied $n$ times to assumption $\phi_1,\dots,\phi_n\vdash\psi$ implies that $\vdash\phi_1\to(\dots\to(\phi_n\to\psi)\dots)$. Thus, $\vdash\Box(\phi_1\to(\dots\to(\phi_n\to\psi)\dots))$, by either the Necessitation inference rule (if $\Box$ is modality $\N$) or Lemma~\ref{K Necessitation} (if $\Box$ is modality $\K$). Hence,  by the Distributivity axiom and the Modus Ponens inference rule, 
$$\vdash\Box\phi_1\to\Box(\phi_2\dots\to(\phi_n\to\psi)\dots).$$ 
Then, $\Box\phi_1\vdash\Box(\phi_2\dots\to(\phi_n\to\psi)\dots)$ by the Modus Ponens inference rule. Thus, again by the Distributivity axiom and the Modus Ponens inference rule, $\Box\phi_1\vdash\Box\phi_2\to \Box(\phi_3\dots\to(\phi_n\to\psi)\dots)$.
Therefore, $\Box\phi_1,\dots,\Box\phi_n\vdash\Box\psi$, by repeating the last two steps $n-2$ times.
\end{proof}


\begin{lemma}[positive introspection]\label{positive introspection lemma}
$\vdash \Box\phi\to\Box\Box\phi$, where $\Box$ is either modality $\N$ or modality $\K_a$. 
\end{lemma}
\begin{proof}
Formula $\Box\neg\Box\phi\to\neg\Box\phi$ is an instance of the Truth axiom. Thus,  by contraposition, $\vdash \Box\phi\to\neg\Box\neg\Box\phi$. Hence, taking into account the following instance of  the Negative Introspection axiom: $\neg\Box\neg\Box\phi\to\Box\neg\Box\neg\Box\phi$,
we have 
\begin{equation}\label{pos intro eq 2}
\vdash \Box\phi\to\Box\neg\Box\neg\Box\phi.
\end{equation}
At the same time, $\neg\Box\phi\to\Box\neg\Box\phi$ is an instance of the Negative Introspection axiom. Thus, $\vdash \neg\Box\neg\Box\phi\to \Box\phi$ by the law of contrapositive in the propositional logic. Hence,  
$\vdash \Box(\neg\Box\neg\Box\phi\to \Box\phi)$
by either the Necessitation inference rule (if $\Box$ is modality $\N$) or Lemma~\ref{K Necessitation} (if $\Box$ is modality $\K$). Thus, by  the Distributivity axiom and the Modus Ponens inference rule, 
$
  \vdash \Box\neg\Box\neg\Box\phi\to \Box\Box\phi.
$
 The latter, together with statement~(\ref{pos intro eq 2}), implies the statement of the lemma by propositional reasoning.
\end{proof}

\begin{lemma}\label{N biconditional lemma}
~
\begin{enumerate}
    \item $\N(\phi\to\psi),\N(\neg\phi\to\neg\psi)\vdash \N(\phi\leftrightarrow \psi)$,
    \item $\N(\phi\to\neg\psi),\N(\neg\phi\to\psi)\vdash \N(\phi\leftrightarrow \neg\psi)$.
\end{enumerate}
\end{lemma}
\begin{proof}
It is provable in the propositional logic that $\phi\to\psi,\neg\phi\to\neg\psi\vdash \phi\leftrightarrow \psi$. Thus, $\N(\phi\to\psi),\N(\neg\phi\to\neg\psi)\vdash \N(\phi\leftrightarrow \psi)$ by Lemma~\ref{superdistributivity lemma}. The proof of the second part of the lemma is similar.
\end{proof}

\begin{lemma}\label{substitution lemma}
~
\begin{enumerate}
    \item $\N(\phi\leftrightarrow\psi), \cN\H_a\phi \vdash \cN\H_a\psi$,
    \item $\N(\phi\leftrightarrow\psi), \cN\S_a\neg\phi\vdash  \cN\S_a\neg\psi$,
    \item $\N(\phi\leftrightarrow\neg\psi), \cN\S_a\neg\phi\vdash  \cN\S_a\psi$.
\end{enumerate}
\end{lemma}
\begin{proof}
Formula $\N(\phi\leftrightarrow\psi)\to(\H_a\phi\to\H_a\psi)$ is an instance of the Substitution axiom. 
Thus, 
$\vdash \N(\phi\leftrightarrow\psi)\to(\neg\H_a\psi\to\neg\H_a\phi)$ by the laws of propositional reasoning.
Hence,
$\N(\phi\leftrightarrow\psi),\neg\H_a\psi\vdash \neg\H_a\phi$
by the Modus Ponens rule applied twice.
Then,
$\N\N(\phi\leftrightarrow\psi),\N\neg\H_a\psi\vdash \N\neg\H_a\phi$
by Lemma~\ref{superdistributivity lemma}.
Thus,
by Lemma~\ref{positive introspection lemma} and the Modus Ponens inference rule,
$\N(\phi\leftrightarrow\psi),\N\neg\H_a\psi\vdash \N\neg\H_a\phi$.
Hence, 
$\N(\phi\leftrightarrow\psi)\vdash\N\neg\H_a\psi\to\N\neg\H_a\phi$ by Lemma~\ref{deduction lemma}.
Then, by the laws of propositional reasoning,
$\N(\phi\leftrightarrow\psi)\vdash\neg\N\neg\H_a\phi\to\neg\N\neg\H_a\psi$.
Thus, by the definition of modality $\cN$, we have
$\N(\phi\leftrightarrow\psi)\vdash\cN\H_a\phi\to\cN\H_a\psi$. Therefore, by the Modus Ponens inference rule,
$\N(\phi\leftrightarrow\psi),\cN\H_a\phi\vdash\cN\H_a\psi$.

To prove the second statement, observe that 
$(\phi\leftrightarrow\psi)\to(\neg\phi\leftrightarrow\neg\psi)$
is a propositional tautology. Thus, 
$\phi\leftrightarrow\psi\vdash \neg\phi\leftrightarrow\neg\psi$
by the Modus Ponens inference rule. 
Hence,
$\N(\phi\leftrightarrow\psi)\vdash \N(\neg\phi\leftrightarrow\neg\psi)$ by Lemma~\ref{superdistributivity lemma}. 
Then, to prove of the second statement, it suffices to show that $\N(\neg\phi\leftrightarrow\neg\psi), \cN\S_a\neg\phi\vdash  \cN\S_a\neg\psi$. The proof of this is the same as of the first statement.

The proof of the third statement is similar to the proof of the second, but it starts with the tautology $(\phi\leftrightarrow\neg\psi)\to(\neg\phi\leftrightarrow\psi)$.
\end{proof}

\begin{lemma}[Lindenbaum]\label{Lindenbaum's lemma}
Any consistent set of formulae can be extended to a maximal consistent set of formulae.
\end{lemma}
\begin{proof}
The standard proof of Lindenbaum's lemma applies here~\cite[Proposition 2.14]{m09}.
\end{proof}

\section{Soundness}\label{Soundness section}

The Truth, the Distributivity, and the Negative Introspection axioms for modalities $\K$ and $\N$ are well-known principles of S5 logic. The soundness of the Knowledge of Necessity axiom follows from Definition~\ref{sat}. Below we prove the soundness of each of the remaining axioms as a separate lemma. We state strong soundness as Theorem~\ref{soundness theorem} in the end of the section. 
In the lemmas below we assume that $w\in W$ is an arbitrary world of an epistemic model with preferences $M=(W,\{\sim_a\}_{a\in\mathcal{A}},\{\prec_a\}_{a\in\mathcal{A}},\pi)$.

\begin{lemma}
If $w\Vdash \E_a\phi$, then $w\Vdash \phi$.
\end{lemma}
\begin{proof}
First, we consider the case $\E=\H$. Note that $w\sim_a w$ because $\sim_a$ is an equivalence relation. Thus, the assumption $w\Vdash \H_a\phi$ implies $w\Vdash \phi$ by item 6(a) of Definition~\ref{sat}. The proof for the case $\E=\S$ is similar, but it uses item 7(a) of Definition~\ref{sat} instead of item 6(a).
\end{proof}

\begin{lemma}
If $w\Vdash \E_a\phi$, then $w\Vdash \K_a\E_a\phi$.
\end{lemma}
\begin{proof}
First, we consider the case $\E=\H$. Consider any world $w'\in W$
 such that $w\sim_a w'$. By item 5 of Definition~\ref{sat}, it suffices to show that $w'\Vdash \H_a\phi$. Indeed, by item 6 of Definition~\ref{sat}, the assumption $w\Vdash \H_a\phi$ of the lemma implies that
\begin{enumerate}
            \item[(a)] $u\Vdash \phi$ for each world $u\in W$ such that $w\sim_a u$,
            \item[(b)] for any two worlds $u,u'\in W$, if $u\nVdash\phi$ and $u'\Vdash\phi$, then $u\prec_a u'$,
            \item[(c)] there is a world $u\in W$ such that $u\nVdash\phi$,
\end{enumerate} 
By the assumption $w\sim_a w'$, statement (a) implies that
\begin{enumerate}
            \item[(a')] $u\Vdash \phi$ for each world $u\in W$ such that $w'\sim_a u$.
\end{enumerate} 
Finally, statements (a'), (b), and (c) imply that  $w'\Vdash \H_a\phi$ by item 6 of Definition~\ref{sat}. The proof for the case $\E=\S$ is similar, but it uses item 7 of Definition~\ref{sat} instead of item 6.
\end{proof}

\begin{lemma}
If $w\Vdash \H_a\phi$, then $w\nVdash \S_a\phi$.
\end{lemma}
\begin{proof}
By item 6(a) of Definition~\ref{sat}, the assumption $w\Vdash \H_a\phi$ implies that $w\Vdash \phi$. By item 6(c) of Definition~\ref{sat}, the same assumption implies that there is a world $w'\in W$ such that $w'\nVdash\phi$. By item 6(b) of Definition~\ref{sat} the assumption $w\Vdash \H_a\phi$ and statements $w'\nVdash \phi$ and $w\Vdash\phi$ imply that $w'\prec_a w$. Thus, $w\nprec_a w'$ because relation $\prec_a$ is a strict partial order. Therefore, $w\nVdash \S_a\phi$ by item 7(b) of Definition~\ref{sat} and statements $w\Vdash \phi$ and $w'\nVdash \phi$.
\end{proof}

\begin{lemma}
If $w\Vdash \cN\E_a\phi$ and $w\Vdash \cN\E_a\psi$, then either $w\Vdash \N(\phi\to\psi)$ or $w\Vdash \N(\psi\to\phi)$.
\end{lemma}
\begin{proof}
First, we consider the case $\E=\H$. Suppose that $w\nVdash \N(\phi\to\psi)$ and $w\nVdash \N(\psi\to\phi)$. Thus, by item 4 of Definition~\ref{sat}, there are epistemic worlds $w_1,w_2\in W$, such that $w_1\nVdash \phi\to\psi$ and $w_2\nVdash \psi\to\phi$. Hence, by item 3 of Definition~\ref{sat},
\begin{equation}\label{nov16-big}
    w_1\Vdash \phi,
    \hspace{5mm}
    w_1\nVdash \psi,
    \hspace{5mm}
    w_2\Vdash \psi,
    \hspace{5mm}
    w_2\nVdash \phi.
\end{equation}
At the same time, by the definition of modality $\cN$ and items 2 and 4 of Definition~\ref{sat}, the assumption $w\Vdash \cN\H_a\phi$ of the lemma  implies that there is a world $w'$ such that $w'\Vdash \H_a\phi$. Hence, $w_2\prec_a w_1$ by item 6(b) of Definition~\ref{sat} and parts $ w_2\nVdash \phi$ and $w_1\Vdash \phi$ of statement~(\ref{nov16-big}).

Similarly, the assumption $w\Vdash \cN\H_a\psi$ of the lemma and parts $ w_1\nVdash \psi$ and $w_2\Vdash \psi$ of statement~(\ref{nov16-big}) imply that $w_1\prec_a w_2$. Note that statements $w_2\prec_a w_1$ and $w_1\prec_a w_2$ are inconsistent because relation $\prec_a$ is a strict partial order. 

The proof in the case $\E=\S$ is similar, but it uses item 7(b) of Definition~\ref{sat} instead of item 6(b).
\end{proof}

\begin{lemma}
If $w\Vdash \cN\H_a\phi$ and $w\Vdash \cN\S_a\psi$, then either $w\Vdash \N(\phi\to\neg\psi)$ or $w\Vdash \N(\neg\psi\to\phi)$.
\end{lemma}
\begin{proof}
Suppose that $w\nVdash \N(\phi\to\neg\psi)$ and $w\nVdash \N(\neg\psi\to\phi)$. Thus, by item 4 of Definition~\ref{sat}, there are epistemic worlds $w_1,w_2\in W$, such that $w_1\nVdash \phi\to\neg\psi$ and $w_2\nVdash \neg\psi\to\phi$. Hence, by item 3 and item 2 of Definition~\ref{sat},
\begin{equation}\label{nov17-big}
    w_1\Vdash \phi,
    \hspace{5mm}
    w_1\Vdash \psi,
    \hspace{5mm}
    w_2\nVdash \psi,
    \hspace{5mm}
    w_2\nVdash \phi.
\end{equation}
At the same time, by the definition of modality $\cN$ and items 2 and 4 of Definition~\ref{sat}, the assumption $w\Vdash \cN\H_a\phi$ of the lemma  implies that there is a world $w'$ such that $w'\Vdash \H_a\phi$. Hence, $w_2\prec_a w_1$ by item 6(b) of Definition~\ref{sat} and parts $ w_2\nVdash \phi$ and $w_1\Vdash \phi$ of statement~(\ref{nov17-big}).

Also, by item 7(b) of Definition~\ref{sat}, the assumption $w\Vdash \cN\S_a\psi$ of the lemma and parts $ w_1\Vdash \psi$ and $w_2\nVdash \psi$ of statement~(\ref{nov17-big}) imply that $w_1\prec_a w_2$. Note that statements $w_2\prec_a w_1$ and $w_1\prec_a w_2$ are inconsistent because relation $\prec_a$ is a strict partial order. 
\end{proof}

\begin{lemma}
If $w\Vdash \E_a\phi$, then $w\nVdash \N\phi$.
\end{lemma}
\begin{proof}
First, suppose that $\E=\H$. Then, by item 6(c) of Definition~\ref{sat}, the assumption $w\Vdash \E_a\phi$ implies that there is an epistemic world $u\in W$ such that $u\nVdash\phi$. Therefore, $w\nVdash \N\phi$ by item 4 of Definition~\ref{sat}. 

The proof in the case $\E=\S$ is similar, but it uses item 7(c) of Definition~\ref{sat} instead of item 6(c).
\end{proof}

\begin{lemma}\label{first emotional predictability sound}
If either $w\Vdash \cN\H_a\phi$ or $w\Vdash \cN\S_a\neg\phi$, then statement $w\Vdash \K_a\phi$ implies $w\Vdash \H_a\phi$.
\end{lemma}
\begin{proof}
First, suppose that $w\Vdash \cN\H_a\phi$. Thus, by the definition of modality $\cN$ and items 2 and 4 of Definition~\ref{sat}, there is an epistemic world $w'\in W$ such that $w'\Vdash \H_a\phi$. Hence, by item 6 of Definition~\ref{sat}, 
\begin{enumerate}
            \item[(a)] $u\Vdash \phi$ for each world $u\in W$ such that $w'\sim_a u$,
            \item[(b)] for any two worlds $u,u'\in W$, if $u\nVdash\phi$ and $u'\Vdash\phi$, then $u\prec_a u'$,
            \item[(c)] there is a world $u\in W$ such that $u\nVdash\phi$.
        \end{enumerate}
Also, by item 5 of Definition~\ref{sat}, the assumption $w\Vdash \K_a\phi$ implies that
\begin{enumerate}
            \item[(a')] $u\Vdash \phi$ for each world $u\in W$ such that $w\sim_a u$.
        \end{enumerate}
By item 6 of Definition~\ref{sat}, statements (a'), (b), and (c) imply $w\Vdash \H_a\phi$.

\vspace{1mm}
Next, suppose that $w\Vdash \cN\S_a\neg\phi$. Thus, by the definition of modality $\cN$ and items 2 and 4 of Definition~\ref{sat}, there is an epistemic world $w'\in W$ such that $w'\Vdash \S_a\neg\phi$. Hence, by item 7 of Definition~\ref{sat}, 
\begin{enumerate}
            \item[(a)] $u\Vdash \neg\phi$ for each world $u\in W$ such that $w'\sim_a u$,
            \item[(b)] for any two worlds $u,u'\in W$, if $u\Vdash\neg\phi$ and $u'\nVdash\neg\phi$, then $u\prec_a u'$,
            \item[(c)] there is a world $u\in W$ such that $u\nVdash\neg\phi$.
        \end{enumerate}
Also, by item 5 of Definition~\ref{sat}, the assumption $w\Vdash \K_a\phi$ implies that
\begin{enumerate}
            \item[(a')] $u\Vdash \phi$ for each world $u\in W$ such that $w\sim_a u$.
        \end{enumerate}
Note that by item 2 of Definition~\ref{sat}, statement (b) implies that
\begin{enumerate}
            \item[(b')] for any two worlds $u,u'\in W$, if $u\nVdash\phi$ and $u'\Vdash\phi$, then $u\prec_a u'$.
        \end{enumerate}
And, by item 2 of Definition~\ref{sat}, statement (a) implies that
\begin{enumerate}
            \item[(c')] $w'\nVdash \phi$
\end{enumerate}
because relation $\sim_a$ is reflexive. Finally, note that by item 6 of Definition~\ref{sat}, statements (a'), (b'), and (c') imply $w\Vdash \H_a\phi$.
\end{proof}

The proof of the next lemma is using the converse models and translation $\tau$ that have been introduced in Definition~\ref{converse model} and Definition~\ref{tau translation} respectively.
\begin{lemma}
If either $w\Vdash \cN\H_a\neg\phi$ or $w\Vdash \cN\S_a\phi$, then statement $w\Vdash \K_a\phi$ implies $w\Vdash \S_a\phi$.
\end{lemma}
\begin{proof}
Let $M^\c$ be the converse model of the epistemic model with preferences $M$ and $\Vdash^\c$ is the satisfaction relation for the model $M^\c$. By Lemma~\ref{first emotional predictability sound}, if either $w\Vdash^\c \cN\H_a\tau(\phi)$ or $w\Vdash^c \cN\S_a\neg\tau(\phi)$, then statement $w\Vdash^c \K_a\tau(\phi)$ implies $w\Vdash^\c \H_a\tau(\phi)$. Thus, by Definition~\ref{tau translation}, if either $w\Vdash^\c \tau(\cN\S_a\phi)$ or $w\Vdash^c \tau(\cN\H_a\neg\phi)$, then statement $w\Vdash^c \tau(\K_a\phi)$ implies $w\Vdash^\c \tau(\S_a\phi)$. Therefore,  by Theorem~\ref{duality theorem}, if either $w\Vdash \cN\S_a\phi$ or $w\Vdash \cN\H_a\neg\phi$, then statement $w\Vdash \K_a\phi$ implies $w\Vdash \S_a\phi$.
\end{proof}

\begin{lemma}
If $w\Vdash \N(\phi\leftrightarrow\psi)$ and $w\Vdash \E_a\phi$, then $w\Vdash \E_a\psi$.
\end{lemma}
\begin{proof}
First, we consider the case $\E=\H$. By item 6 of Definition~\ref{sat}, the assumption $w\Vdash \H_a\phi$ implies that 
\begin{enumerate}
            \item[(a)] $u\Vdash \phi$ for each world $u\in W$ such that $w\sim_a u$,
            \item[(b)] for any two worlds $u,u'\in W$, if $u\nVdash\phi$ and $u'\Vdash\phi$, then $u\prec_a u'$,
            \item[(c)] there is a world $u\in W$ such that $u\nVdash\phi$.
        \end{enumerate}
Thus, by the assumption $w\Vdash \N(\phi\leftrightarrow\psi)$ and item 4 of Definition~\ref{sat},
\begin{enumerate}
            \item[(a')] $u\Vdash \psi$ for each world $u\in W$ such that $w\sim_a u$,
            \item[(b')] for any two worlds $u,u'\in W$, if $u\nVdash\psi$ and $u'\Vdash\psi$, then $u\prec_a u'$,
            \item[(c')] there is a world $u\in W$ such that $u\nVdash\psi$.
        \end{enumerate}
By item 6 of Definition~\ref{sat}, statements (a'), (b'), and (c') imply $w\Vdash \H_a\psi$.

The proof in the case $\E=\S$ is similar, but it uses item 7 of Definition~\ref{sat} instead of item 6.
\end{proof}

The strong soundness theorem below follows from the lemmas proven above.

\begin{theorem}\label{soundness theorem}
For any epistemic world $w$ of an epistemic model with preferences, any set of formulae $X\subseteq \Phi$, and any formula $\phi\in\Phi$, if $w\Vdash\chi$ for each formula $\chi\in X$ and $X\vdash\phi$, then $w\Vdash\phi$. \qed
\end{theorem}

\section{Utilitarian Emotions}\label{Utilitarian Emotions}

Lang, van er Torre, and Weydert introduced a notion of utilitarian desire which is based on a utility function rather than a preference relation~\cite{lvw02aamas}. Although desire, as an emotion, is different from the happiness and sadness emotions that we study in this article, their approach could be adopted to happiness and sadness as well. To do this, one needs to modify Definition~\ref{epistemic model with preferences} to include agent-specific utility functions instead of agent-specific preference relations: 

\begin{definition}\label{epistemic model with utilities}
A tuple $(W,\{\sim_a\}_{a\in\mathcal{A}},\{u_a\}_{a\in\mathcal{A}},\pi)$ is called an epistemic model with utilities if
\begin{enumerate}
    \item $W$ is a set of epistemic worlds,
    \item $\sim_a$ is an ``indistinguishability'' equivalence relation on set $W$ for each agent $a\in\mathcal{A}$,
    \item $u_a$ is a ``utility'' function from set $W$ to real numbers for each agent $a\in\mathcal{A}$,
    \item $\pi(p)$ is a subset of $W$ for each propositional variable $p$.
\end{enumerate}
\end{definition}

Below is the definition of the satisfaction relation for epistemic model with utilities. Its parts 6(b) and 7(b) are similar to Lang, van er Torre, and Weydert utilitarian desire definition in \cite{lvw02aamas}. Unlike the current article, \cite{lvw02aamas} does not prove any completeness results. In the definition below we assume that language $\Phi$ is modified to incorporate a no-negative real ``degree'' parameter into modalities $\H^d_a$ and $\S^d_a$. We read statement $\H^d_a\phi$ as ``agent $a$ is happy about $\phi$ with degree $d$''. Similarly, we read $\S^d_a\phi$ as ``agent $a$ is sad about $\phi$ with degree $d$''.

\begin{definition}\label{sat-u}
For any epistemic model with utilities $(W,\{\sim_a\}_{a\in\mathcal{A}},\{u_a\}_{a\in\mathcal{A}},\pi)$, any world $w\in W$, and any formula $\phi\in\Phi$, satisfaction relation $w\Vdash\phi$ is defined as follows:
\begin{enumerate}
    \item $w\Vdash p$, if $w\in \pi(p)$,
    \item $w\Vdash\neg\phi$, if $w\nVdash\phi$,
    \item $w\Vdash\phi\to\psi$, if $w\nVdash\phi$ or $w\Vdash\psi$,
    \item $w\Vdash\N\phi$, if $v\Vdash \phi$ for each world $v\in W$,
    \item $w\Vdash\K_a\phi$, if $v\Vdash \phi$ for each world $v\in W$ such that $w\sim_a v$,
    \item $w\Vdash \H^d_a\phi$, if the following three conditions are satisfied:
        \begin{enumerate}
            \item $v\Vdash \phi$ for each world $v\in W$ such that $w\sim_a v$,
            \item for any $v,v'\in W$, if $v\nVdash\phi$ and $v'\Vdash\phi$, then $u_a(v)+d\le u_a(v')$,
            \item there is a world $v\in W$ such that $v\nVdash\phi$,
        \end{enumerate}
    \item $w\Vdash \S^d_a\phi$, if the following three conditions are satisfied:
        \begin{enumerate}
            \item $v\Vdash \phi$ for each world $v\in W$ such that $w\sim_a u$,
            \item for any worlds $v,v'\in W$, if $v\Vdash\phi$ and $v'\nVdash\phi$, then $u_a(v)+d\le u_a(v')$,
            \item there is a world $u\in W$ such that $u\nVdash\phi$.
        \end{enumerate}
\end{enumerate}
\end{definition}

We have already defined utility functions for our Battle of Cuisines scenario, see Section~\ref{The Battle of Cuisines Scenario section}. In the two propositions below we use this scenario to illustrate utilitarian happiness modality. Note how Sanaz is much happier to be with Pavel in a Russian restaurant than she is to be with him in a restaurant.

\begin{proposition}\label{oct27 Hs}
$(x,x)\Vdash \H^1_s(\mbox{``Sanaz and Pavel are in the same restaurant''})$,
where $x\in \{\mbox{Iranian},\mbox{Russian}\}$.
\end{proposition}
\begin{proof}
We verify conditions (a), (b), and (c) from item 6 of Definition~\ref{sat-u}:

\vspace{1mm}
\noindent{\em Condition a:} Consider any epistemic world $(y,z)$ such that $(x,x)\sim_s (y,z)$. It suffices to show that $(y,z)\Vdash \mbox{``Sanaz and Pavel are in the same restaurant''}$. The latter is true because assumption $(x,x)\sim_s (y,z)$ implies that $x=y$ and $x=z$ in the  perfect information setting of the Battle of Cuisines scenario.

\vspace{1mm}
\noindent{\em Condition b:} Consider any worlds $(y,z)$ and $(y',z')$ such that
\begin{eqnarray}
    (y,z)\nVdash \mbox{``Sanaz and Pavel are in the same restaurant''},\label{nov2-a}\\
    (y',z')\Vdash \mbox{``Sanaz and Pavel are in the same restaurant''}.\label{nov2-b}
\end{eqnarray}
Statement~(\ref{nov2-a}) implies that $u_s(y,z)=0$, see Table~\ref{example table figure}. Similarly, statement~(\ref{nov2-b}) implies that $u_s(y',z')\ge 1$. Therefore, $u_s(y,z)+1=1\le u_s(y',z')$.

\vspace{1mm}
\noindent{\em Condition c:}
$(\mathrm{Russian},\mathrm{Iranian})\nVdash\mbox{``Sanaz and Pavel are in the same restaurant''}$.
\end{proof}

\begin{proposition}\label{nov2 Hs}
$$(\mathrm{Russian},\mathrm{Russian})\Vdash \H^2_s(\mbox{``Sanaz and Pavel are in the Russian restaurant''}).$$
\end{proposition}
\begin{proof}
Conditions (a) and (c) from item 6 of Definition~\ref{sat-u} could be verified similarly to the proof of Proposition~\ref{oct27 Hs}. Below we verify condition (b).

Consider any worlds $(y,z)$ and $(y',z')$ such that
\begin{eqnarray}
    (y,z)\nVdash \mbox{``Sanaz and Pavel are in the Russian restaurant''},\label{nov2-c}\\
    (y',z')\Vdash \mbox{``Sanaz and Pavel are in the Russian restaurant''}.\label{nov2-d}
\end{eqnarray}
Statement~(\ref{nov2-c}) implies that $u_s(y,z)\le 1$, see Table~\ref{example table figure}. Similarly, statement~(\ref{nov2-d}) implies that $u_s(y',z')= 3$. Therefore, $u_s(y,z)+2\le 1 + 2 = 3= u_s(y',z')$.
\end{proof}

\section{Goodness-Based Emotions}\label{Goodness-Based Emotions}

Lorini and Schwarzentruber proposed a different framework for defining emotions~\cite{ls11ai}. Instead of specifying preference relations on the epistemic worlds they label some of worlds as desirable or ``good''. In such a setting they define modalities ``rejoice'' and ``disappointment'' that are similar to our modalities ``happiness'' and ``sadness''. In this section we compare their approach to ours. Although their framework endows agents with actions, it appears that actions are essential for defining regret and are less important for capturing rejoice and disappointment. In the definition below, we simplify Lorini and Schwarzentruber's framework to action-less models that we call epistemic model with goodness.

\begin{definition}\label{epistemic model with goodness}
A tuple $(W,\{\sim_a\}_{a\in\mathcal{A}},\{G_a\}_{a\in\mathcal{A}},\pi)$ is called an epistemic model with goodness if
\begin{enumerate}
    \item $W$ is a set of epistemic worlds,
    \item $\sim_a$ is an ``indistinguishability'' equivalence relation on set $W$ for each agent $a\in\mathcal{A}$,
    \item $G_a\subseteq W$ is a nonempty set of ``good'' epistemic worlds for agent $a\in\mathcal{A}$,
    \item $\pi(p)$ is a subset of $W$ for each propositional variable $p$.
\end{enumerate}
\end{definition}

To represent the gift example from Figure~\ref{example gift figure} as an epistemic model with goodness, we need to specify the sets  of good epistemic worlds $G_s$ and $G_p$ of Sanaz and Pavel. 
A natural way to do this is to assume that $G_s=G_p=\{w\}$. In other words, the ideal outcome for both of them would be if the gift and the card arrives to the recipients. 

In the lottery example, the desirable outcome for each agent is when the agents wins the lottery. In other words, $G_s=\{u\}$, $G_p=\{w\}$, and $G_o=\{v\}$, see Figure~\ref{example lottery figure}.

In the Battle of Cuisines example captured in Table~\ref{example table figure}, the choice of good epistemic worlds is not obvious. On one hand, we can assume that good worlds for both Sanaz and Pavel are the ones where they have positive pay-offs. In this case, $G_s=G_p=\{(\mathrm{Iranian},\mathrm{Iranian}),(\mathrm{Russian},\mathrm{Russian})\}$. Alternatively, we can choose the good worlds to be those where they get maximal pay-off. In that case, 
$G_s=\{(\mathrm{Russian},\mathrm{Russian})\}$ 
and 
$G_p=\{(\mathrm{Iranian},\mathrm{Iranian})\}$.
Note that our epistemic models with preferences approach provides a more fine-grained semantics that does not force the choice between these two alternatives.

In the definition below, we rephrase Lorini and Schwarzentruber's formal definitions of ``rejoice'' and ``disappointment'' in terms of epistemic models with goodness. We denote modalities ``rejoice'' and ``disappointment'' by $\H$ and $\S$ respectively to be consistent with the notations in the rest of this article. 

\begin{definition}\label{sat with goodness}
For any world $w\in W$ of an epistemic model with goodness $(W,\{\sim_a\}_{a\in\mathcal{A}},\{G_a\}_{a\in\mathcal{A}},\pi)$ and any formula $\phi\in\Phi$, satisfaction relation $w\Vdash\phi$ is defined as follows:
\begin{enumerate}
    \item $w\Vdash p$, if $w\in \pi(p)$,
    \item $w\Vdash\neg\phi$, if $w\nVdash\phi$,
    \item $w\Vdash\phi\to\psi$, if $w\nVdash\phi$ or $w\Vdash\psi$,
    \item $w\Vdash\N\phi$, if $u\Vdash \phi$ for each world $u\in W$,
    \item\label{item K} $w\Vdash\K_a\phi$, if $u\Vdash \phi$ for each world $u\in W$ such that $w\sim_a u$,
    \item\label{item H} $w\Vdash \H_a\phi$, if the following three conditions are satisfied:
        \begin{enumerate}
            \item $u\Vdash \phi$ for each world $u\in W$ such that $w\sim_a u$,
            \item $u\Vdash\phi$ for each world $u \in G_a$,
            \item there is a world $u\in W$ such that $u\nVdash\phi$,
        \end{enumerate}
    \item $w\Vdash \S_a\phi$, if the following three conditions are satisfied:
        \begin{enumerate}
            \item $u\Vdash \phi$ for each world $u\in W$ such that $w\sim_a u$,
            \item $u\nVdash\phi$ for each world $u \in G_a$,
            \item there is a world $u\in W$ such that $u\nVdash\phi$.
        \end{enumerate}
\end{enumerate}
\end{definition}

Consider the discussed above epistemic model with goodness for the gift scenario in which $G_s=G_p=\{w\}$.
It is relatively easy to see that all propositions that we proved in Section~\ref{gift scenario section} for preference-based semantics hold true under goodness-based semantics of modalities $\H$ and $\S$ given in Definition~\ref{sat with goodness}. 

The situation is different for the Battle of Cuisines scenario. If $\Vdash_1$ denotes the satisfaction relation of the epistemic model with goodness where $G_s=G_p=\{(\mathrm{Iranian},\mathrm{Iranian}),(\mathrm{Russian},\mathrm{Russian})\}$, then the following two propositions are true just like they are under our definition of happiness (see Proposition~\ref{sept25 Hs} and Proposition~\ref{oct19-rr}):

\begin{proposition}\label{oct19-rr2}
$$(\mathrm{Russian},\mathrm{Russian})\Vdash_1 \H_s(\mbox{``Sanaz and Pavel are in the same restaurant''}).$$
\end{proposition}
\begin{proof}
It suffices to verify conditions (a), (b), and (c) of item 6 of Definition~\ref{sat with goodness}:

\vspace{1mm}
\noindent{\em Condition a:} Since the Battle of the Cuisines is a setting with perfect information, it suffices to show that $$(\mathrm{Russian},\mathrm{Russian})\Vdash_1 (\mbox{``Sanaz and Pavel are in the same restaurant''}).$$
The last statement is true by the definition of the world $(\mathrm{Russian},\mathrm{Russian})$.

\vspace{1mm}
\noindent{\em Condition b:} Statement ``Sanaz and Pavel are in the same restaurant'' is satisfied in both good worlds: $(\mathrm{Iranian},\mathrm{Iranian})$ and $(\mathrm{Russian},\mathrm{Russian})$.

\vspace{1mm}
\noindent{\em Condition c:} Statement ``Sanaz and Pavel are in the same restaurant'' is not satisfied in the world $(\mathrm{Russian},\mathrm{Iranian})$.
\end{proof}

\begin{proposition}
$$(\mathrm{Russian},\mathrm{Russian})\nVdash_1 \H_s(\mbox{``Sanaz is in the Russian restaurant''}).$$
\end{proposition}
\begin{proof}
Note that $(\mathrm{Iranian},\mathrm{Iranian})$ is a good world, in which statement ``Sanaz is in the Russian restaurant'' is not satisfied. Therefore, the statement of the proposition is true by item 6(b) of Definition~\ref{sat with goodness}.
\end{proof}

However, for the satisfaction relation $\Vdash_2$ of the epistemic model with goodness where  $G_s=\{(\mathrm{Russian},\mathrm{Russian})\}$ 
and 
$G_p=\{(\mathrm{Iranian},\mathrm{Iranian})\}$, the situation is different:

\begin{proposition}\label{oct20-a}
$$(\mathrm{Russian},\mathrm{Russian})\Vdash_2 \H_s(\mbox{``Sanaz and Pavel are in the same restaurant''}).$$
\end{proposition}
\begin{proof}
The proof  is similar to the proof of Proposition~\ref{oct19-rr2} except that in {\em Condition b} we only need to consider the world $(\mathrm{Russian},\mathrm{Russian})$.
\end{proof}

\begin{proposition}\label{oct20-b}
$$(\mathrm{Russian},\mathrm{Russian})\Vdash_2 \H_s(\mbox{``Sanaz is in the Russian restaurant''}).$$
\end{proposition}
\begin{proof}
It suffices to verify conditions (a), (b), and (c) of item 6 of Definition~\ref{sat with goodness}:

\vspace{1mm}
\noindent{\em Condition a:} Since the Battle of the Cuisines is a setting with perfect information, it suffices to show that $$(\mathrm{Russian},\mathrm{Russian})\Vdash_2 (\mbox{``Sanaz is in the Russian restaurant''}).$$
The last statement is true by the definition of the world $(\mathrm{Russian},\mathrm{Russian})$.

\vspace{1mm}
\noindent{\em Condition b:} Note that in the current setting set $G_s$ contains only element $(\mathrm{Russian},\mathrm{Russian})$ and that $$(\mathrm{Russian},\mathrm{Russian})\Vdash_2 (\mbox{``Sanaz is in the Russian restaurant''}).$$

\vspace{1mm}
\noindent{\em Condition c:} $(\mathrm{Iranian},\mathrm{Russian})\nVdash_2 (\mbox{``Sanaz is in the Russian restaurant''})$.
\end{proof}

We conclude this section by an observation that the Coherence of Potential Emotions axiom is not universally true under goodness-based semantics. Indeed, note that according to Proposition~\ref{oct20-a} and Proposition~\ref{oct20-b}, there is an epistemic worlds in which Sanaz is happy that ``Sanaz is in the Russian restaurant'' and there is an epistemic world in which she is happy that ``Sanaz and Pavel are in the same restaurant''. If the Coherence of Potential Emotions axiom holds in this setting, then one of these statements would imply the other, but neither of them does.

\section{Canonical Model}

In the rest of this article we prove strong completeness of our logical system with respect to the semantics given in Definition~\ref{sat}. As usual, the proof of the completeness is based on a construction of a canonical model. In this section, for any maximal consistent set of formulae $X\subseteq\Phi$, we define canonical epistemic  model with preferences $M(X_0)=(W,\{\sim_a\}_{a\in\mathcal{A}},\{\prec_a\}_{a\in\mathcal{A}},\pi)$.

As it is common in modal logic, we define worlds as maximal consistent sets of formulae. Since the meaning of modality $\N$ in our system is ``for all worlds'', we require all worlds in the canonical model to have the same $\N$-formulae. We achieve this through the following definition.

\begin{definition}\label{canonical world definition}
$W$ is the set of all such maximal consistent sets of formulae $Y$ that $\{\phi\in\Phi\;|\;\N\phi\in X_0\}\subseteq Y$.
\end{definition}

Note that although the above definition only requires all $\N$-formulae from set $X_0$ to be in set $Y$, it is possible to show that the converse is also true due to the presence  of the Negative Introspection axiom for modality $\N$ in our system. 

\begin{lemma}\label{X0 in W lemma}
$X_0\in W$.
\end{lemma}
\begin{proof}
Consider any formula $\N\phi\in X_0$. By Definition~\ref{canonical world definition}, it suffices to show that $\phi\in X_0$. Indeed, the assumption $\N\phi\in X_0$ implies that $X_0\vdash\phi$ by the Truth axiom and the Modus Ponens inference rule. Therefore, $\phi\in X_0$ because set $X_0$ is maximal.
\end{proof}

\begin{definition}\label{canonical sim}
For any worlds $w,u\in W$, let $w\sim_a u$ if $\phi\in u$ for each formula $\K_a\phi\in w$.
\end{definition}

Alternatively, one can define $w\sim_a u$ if sets $w$ and $u$ have the same $\K_a$-formulae. Our approach results in shorter proofs, but it requires to prove the following lemma.

\begin{lemma}\label{canonical sim is equivalence relation lemma}
Relation $\sim_a$ is an equivalence relation on set $W$.
\end{lemma}
\begin{proof} {\bf Reflexivity:} Consider any formula $\phi\in\Phi$. Suppose that $\K_a\phi\in w$. By Definition~\ref{canonical sim}, it suffices to show that $\phi\in w$. Indeed, assumption $\K_a\phi\in w$ implies $w\vdash\phi$ by the Truth axiom and the Modus Ponens inference rule. Therefore, $\phi\in w$ because set $w$ is maximal.

\vspace{1mm}
\noindent {\bf Symmetry:} Consider any epistemic worlds $w, u \in W$ such that $w \sim_a u$ and any formula $\K_a\phi \in u$. By Definition~\ref{canonical sim}, it suffices to show  $\phi \in w$. Suppose the opposite. Then, $\phi \notin w$. Hence, $w \nvdash \phi$ because set $w$ is maximal. Thus, $w \nvdash \K_a \phi$ by the contraposition of the Truth axiom. Then, $\neg \K_a \phi \in w$ because set $w$ is maximal. Thus, $w \vdash \K_a \neg \K_a \phi$ by the Negative Introspection axiom and the Modus Ponens inference rule. Hence, $\K_a \neg \K_a \phi\in w$ because set $w$ is maximal. Then, $\neg \K_a \phi\in u$ by assumption $w\sim_a u$ and Definition~\ref{canonical sim}. Therefore, $\K_a\phi \notin u$ because set $w$ is consistent, which contradicts the assumption  $\K_a\phi \in u$. 

\vspace{1mm}
\noindent {\bf Transitivity:} Consider any epistemic worlds $w, u, v \in W$ such that $w \sim_a u$ and $u \sim_a v$ and any formula $\K_a\phi\in w$. By Definition~\ref{canonical sim}, it suffices to show  $\phi\in v$.  Assumption $\K_a\phi\in w$ implies $w \vdash \K_a \K_a \phi$ by Lemma~\ref{positive introspection lemma} and the Modus Ponens inference rule. Thus, $\K_a \K_a \phi \in w$ because set $w$ is maximal. Hence, $\K_a \phi \in u$ by the assumption $w \sim_a u$ and Definition~\ref{canonical sim}. Therefore, $\phi \in v$ by the assumption $u \sim_a v$ and Definition~\ref{canonical sim}.
\end{proof}

The next step in specifying the canonical model is to define preference relation $\prec_a$ for each agent $a\in\mathcal{A}$, which we do in Definition~\ref{canonical prec}. Towards this definition, we first introduce the ``emotional base'' $\Delta_a$ for each agent $a$. The set $\Delta_a$ contains a formula $\delta$ if agent $a$ could either be potentially happy about $\delta$ or potentially sad about $\neg\phi$.

\begin{definition}\label{Delta a}
$
\Delta_a = \{\delta\in\Phi\;|\; \cN\H_a\delta\in X_0\}\cup \{\delta\in\Phi\;|\; \cN\S_a\neg\delta\in X_0\}
$.
\end{definition}

The next lemma holds because set $\Phi$ is countable.
\begin{lemma}\label{Delta_a is countable}
Set $\Delta_a$ is countable for each agent $a\in\mathcal{A}$. \qed
\end{lemma}

Next, we introduce a total pre-order $\sqsubseteq_a$ on emotional base $\Delta_a$ of each agent $a\in \mathcal{A}$. Note that this pre-order is different from canonical preference relation $\prec_a$ that we introduce in Definition~\ref{canonical prec}.

\begin{definition}\label{sqsubseteq definition}
For any agent $a\in\mathcal{A}$ and any two formulae $\delta,\delta'\in\Delta_a$, let $\delta\sqsubseteq \delta'$ if $\N(\delta\to\delta')\in X_0$.
\end{definition}

\begin{lemma}\label{sqsubseteq contraposition lemma}
$\N(\neg\delta'\to\neg\delta)\in X_0$
for any agent $a\in\mathcal{A}$ and any two formulae $\delta,\delta'\in\Delta_a$ such that $\delta\sqsubseteq\delta'$.
\end{lemma}
\begin{proof}
Formula $(\delta\to\delta')\to(\neg\delta'\to\neg\delta)$ is a propositional tautology. Thus, $\vdash \N((\delta\to\delta')\to(\neg\delta'\to\neg\delta))$ by the Necessitation inference rule. Hence,  
\begin{equation}\label{contraposition under N}
    \vdash\N(\delta\to\delta')\to\N(\neg\delta'\to\neg\delta)
\end{equation}
by the Distributivity axiom and the Modus Ponens inference rule.

Suppose that $\delta\sqsubseteq\delta'$. Thus, $\N(\delta\to\delta')\in X_0$ by Definition~\ref{sqsubseteq definition}. Hence, $X_0\vdash \N(\neg\delta'\to\neg\delta)$ by statement~(\ref{contraposition under N}) and the Modus Ponens inference rule. Therefore, $\N(\neg\delta'\to\neg\delta)\in X_0$ because set $X_0$ is maximal.
\end{proof}

\begin{lemma}\label{sqsubseteq total pre-order lemma}
For any agent $a\in\mathcal{A}$, relation $\sqsubseteq$ is a total pre-order on set $\Delta_a$.
\end{lemma}
\begin{proof}
We need to show that relation $\sqsubseteq$ is reflexive, transitive, and total.

\vspace{1mm}
\noindent{\bf Reflexivity:} Consider an arbitrary formula $\delta\in\Phi$. By Definition~\ref{sqsubseteq definition}, it suffices to show that $\N(\delta\to\delta)\in X_0$. Indeed, formula $\delta\to\delta$ is a propositional tautology. Thus, $\vdash \N(\delta\to\delta)$ by the Necessitation inference rule. Therefore, $\N(\delta\to\delta)\in X_0$ because set $X_0$ is maximal.

\vspace{1mm}
\noindent{\bf Transitivity:} Consider arbitrary $\delta_1,\delta_2,\delta_3\in\Phi$ such that $\N(\delta_1\to\delta_2)\in X_0$ and $\N(\delta_2\to\delta_3)\in X_0$. By Definition~\ref{sqsubseteq definition}, it suffices to show that $\N(\delta_1\to\delta_3)\in X_0$. Indeed, note that formula 
$
(\delta_1\to\delta_2)\to ((\delta_2\to\delta_3)\to(\delta_1\to\delta_3) )
$
is a propositional tautology. Thus, 
$
\vdash \N((\delta_1\to\delta_2)\to ((\delta_2\to\delta_3)\to(\delta_1\to\delta_3) ))
$
by the Necessitation inference rule. Hence, by the Distributivity axiom and the Modus Ponens inference rule,
$
\vdash \N(\delta_1\to\delta_2)\to \N((\delta_2\to\delta_3)\to(\delta_1\to\delta_3) )
$. 
Then,
$
X_0 \vdash \N((\delta_2\to\delta_3)\to(\delta_1\to\delta_3) )
$
by the assumption $\N(\delta_1\to\delta_2)\in X_0$ and the Modus Ponens inference rule.
Thus,
$
X_0 \vdash \N(\delta_2\to\delta_3)\to\N(\delta_1\to\delta_3)
$
by the Distributivity axiom and the Modus Ponens inference rule. Hence,
by the assumption $\N(\delta_2\to\delta_3)\in X_0$ and the Modus Ponens rule,
$
X_0 \vdash \N(\delta_1\to\delta_3)
$.
Therefore, $\N(\delta_1\to\delta_3)\in X_0$ because set $X_0$ is maximal.

\vspace{1mm}
\noindent{\bf Totality:} Consider arbitrary formulae $\delta_1,\delta_2\in \Delta_a$. By Definition~\ref{sqsubseteq definition}, it suffices to show that either $\N(\delta_1\to\delta_2)\in X_0$ or $\N(\delta_2\to\delta_1)\in X_0$. By Definition~\ref{Delta a}, without loss of generality, we can assume that on the the following three cases takes place:

\vspace{1mm}
\noindent{\em Case I:} $\cN\H_a\delta_1, \cN\H_a\delta_2\in X_0$. Then, $X_0\vdash \N(\delta_1\to\delta_2)\vee \N(\delta_2\to\delta_1)$ by the first Coherence of Potential Emotions axiom and propositional reasoning. Therefore, because set $X_0$ is maximal, either $\N(\delta_1\to\delta_2)\in X_0$ or $\N(\delta_2\to\delta_1)\in X_0$. 

\vspace{1mm}
\noindent{\em Case II:} $\cN\S_a\neg\delta_1, \cN\S_a\neg\delta_2\in X_0$. Similarly to the previous case, we can show that either $\N(\neg\delta_1\to\neg\delta_2)\in X_0$ or $\N(\neg\delta_2\to\neg\delta_1)\in X_0$. Without loss of generality, suppose that $\N(\neg\delta_1\to\neg\delta_2)\in X_0$. Note that formula
$
(\neg\delta_1\to\neg\delta_2)\to(\delta_2\to\delta_1)
$
is a propositional tautology. Thus,
$
\vdash\N((\neg\delta_1\to\neg\delta_2)\to(\delta_2\to\delta_1))
$
by the Necessitation inference rule. Hence, 
$
\vdash\N(\neg\delta_1\to\neg\delta_2)\to\N(\delta_2\to\delta_1)
$ by the Distributivity axiom and the Modus Ponens rule. Thus,
$
X_0\vdash\N(\delta_2\to\delta_1)
$ 
by the assumption $\N(\neg\delta_1\to\neg\delta_2)\in X_0$.
Therefore, $\N(\delta_2\to\delta_1)\in X_0$ because set $X_0$ is maximal.

\vspace{1mm}
\noindent{\em Case III:} $\cN\H_a\delta_1,\cN\S_a\neg\delta_2\in X_0$. Thus, by the second Coherence of Potential Emotions axiom and propositional reasoning,
$
X_0\vdash \N(\delta_1\to\neg\neg\delta_2) \vee \N(\neg\neg\delta_2\to \delta_1)
$.
Hence, 
either $\N(\delta_1\to\neg\neg\delta_2)\in X_0$ or $\N(\neg\neg\delta_2\to \delta_1)\in X_0$ because set $X_0$ is consistent. Then using an argument similar to the one in Case II and propositional tautologies
$$
(\delta_1\to\neg\neg\delta_2) \to (\delta_1\to\delta_2)
\mbox{ and }
(\neg\neg\delta_2\to \delta_1) \to (\delta_2\to \delta_1)
$$
one can conclude that either $\N(\delta_1\to\delta_2)\in X_0$ or $\N(\delta_2\to \delta_1)\in X_0$. 
\end{proof}

We now are ready to define  preference relation $\prec_a$ on epistemic worlds of the canonical model.  

\begin{definition}\label{canonical prec}
$w\prec_a u$ if there is a formula $\delta\in \Delta_a$ such that $\delta\notin w$ and $\delta\in u$. 
\end{definition}

Note that the transitivity of the relation $\prec_a$ is not obvious. We prove it as a part of the next lemma.

\begin{lemma}
$\prec_a$ is a strict partial order on $W$.
\end{lemma}
\begin{proof}
{\bf Irreflexivity}. Suppose that $w\prec_a w$ for some world $w\in W$. Thus, by Definition~\ref{canonical prec}, there exists formula $\delta\in\Delta_a$ such that $\delta\notin w$ and $\delta\in w$, which is a contradiction.


\noindent{\bf Transitivity}. Consider any worlds $w,u,v\in W$ such that $w\prec_a u$ and $u\prec_a v$. It suffices to prove that $w\prec_a v$. Indeed, by Definition~\ref{canonical prec}, assumptions $w\prec_a u$ and $u\prec_a v$ imply that there are formulae $\delta_1,\delta_2$ in $\Delta_a$ such that 
\begin{equation}\label{another four facts}
    \delta_1\notin w, \;\;\;
    \delta_1\in u, \;\;\;
    \delta_2\notin u, \;\;\; \mbox{and }\;\;\;
    \delta_2\in v.
\end{equation}
By Lemma~\ref{sqsubseteq total pre-order lemma},  either $\delta_1\sqsubseteq \delta_2$ or $\delta_2\sqsubseteq \delta_1$. We consider these two cases separately.

\vspace{1mm}
\noindent{\em Case I:} $\delta_1\sqsubseteq \delta_2$. Then, $\N(\delta_1\to\delta_2)\in X_0$ by Definition~\ref{canonical prec}. Hence, $\delta_1\to\delta_2\in u$ by Definition~\ref{canonical world definition}. Thus, $u\vdash \delta_2$ by the part $\delta_1\in u$ of statement~(\ref{another four facts}) and the Modus Ponens inference rule. Therefore, $\delta_2\in u$ because set $u$ is maximal, which contradicts the part $\delta_2\notin u$ of statement~(\ref{another four facts}).

\vspace{1mm}
\noindent{\em Case II:} $\delta_2\sqsubseteq \delta_1$. Then, $\N(\delta_2\to\delta_1)\in X_0$ by Definition~\ref{canonical prec}. Thus, $\delta_2\to\delta_1\in v$ by Definition~\ref{canonical world definition}. Hence, $v\vdash \delta_1$ by the part $\delta_2\in v$ of statement~(\ref{another four facts}) and the Modus Ponens inference rule. Then, $\delta_1\in v$ because set $v$ is maximal. Therefore, $w\prec_a v$ by Definition~\ref{canonical prec} and the part $\delta_1\notin w$ of statement~(\ref{another four facts}).
\end{proof}

\begin{definition}\label{canonical pi}
$\pi(p)=\{w\in W\;|\;p\in w\}$.
\end{definition}

This concludes the definition of the canonical epistemic  model with preferences $M(X_0)=(W,\{\sim_a\}_{a\in\mathcal{A}},\{\prec_a\}_{a\in\mathcal{A}},\pi)$.

\section{Properties of a Canonical Model}

As usual the proof of the completeness is centered around an ``induction'' or ``truth'' lemma. In our case, this is Lemma~\ref{induction lemma}. We precede this lemma with several auxiliary lemmas that are used in the induction step of the proof of Lemma~\ref{induction lemma}. For the benefit of the reader, we grouped these auxiliary lemmas into several subsections. Throughout this section up to and including Lemma~\ref{induction lemma}, we assume a fixed canonical model $M(X_0)$.

\subsection{Properties of Modality $\N$}

\begin{lemma}\label{N child all}
For any worlds $w,u\in W$ and any formula $\phi\in\Phi$, if $\N\phi\in w$, then $\phi\in u$.
\end{lemma}
\begin{proof}
Suppose that $\phi \notin u$.
Thus, $u\nvdash \phi$ because set $u$ is maximal. 
Hence, $\N\phi\notin u$ by the Truth axiom. Then,
$\N\N\phi\notin X_0$
by Definition~\ref{canonical world definition}.
Thus, $X_0\nvdash\N\N\phi$ because set $X_0$ is maximal.
Hence, $\N\phi\notin X_0$ by Lemma~\ref{positive introspection lemma}. Then, $\neg\N\phi\in X_0$ because set $X_0$ is maximal. Thus,
$X_0\vdash \N\neg\N\phi$ by the Negative Introspection axiom  and the Modus Ponens inference rule. Hence, $\N\neg\N\phi\in X_0$ because set $X_0$ is maximal. Then, $\neg\N\phi\in w$ by Definition~\ref{canonical world definition}. Therefore, $\N\phi\notin w$ because set $w$ is consistent.
\end{proof}

\begin{lemma}\label{N child exists}
For any world $w\in W$ and any formula $\phi\in\Phi$, if $\N\phi\notin w$, then there is a world $u\in W$ such that $\phi\notin u$.
\end{lemma}
\begin{proof}
Consider the set 
$
X=\{\neg\phi\}\cup\{\psi\;|\;\N\psi\in X_0\}
$.
We start by showing that set $X$ is consistent. Suppose the opposite. Then, there are formulae $ \N\psi_1,\dots, \N\psi_n \in X_0 $ such that 
$
\psi_1,\dots,\psi_n\vdash \phi
$.
Thus,
$
\N\psi_1,\dots,\N\psi_n\vdash \N\phi
$
by Lemma~\ref{superdistributivity lemma}.
Hence,
$
X_0\vdash \N\phi
$ 
because $ \N\psi_1,\dots, \N\psi_n \in X_0 $. Then,
$
X_0\vdash \N\N\phi
$ 
by Lemma~\ref{positive introspection lemma} and the Modus Ponens inference rule.
Thus, 
$ 
\N\N\phi \in X_0 
$
because set $X_0$ is maximal. Hence,
$
\N\phi\in w
$
by Definition~\ref{canonical world definition} which contradicts the assumption $\N\phi\notin w$ of the lemma. Therefore, set $X$ is consistent. 

Let set $u$ be a maximum consistent extension of set $X$. Such set exists by Lemma~\ref{Lindenbaum's lemma}. Note that $u \in W $ by Definition~\ref{canonical world definition} and the choice of sets $X$ and $u$. Also, $\neg\phi\in X\subseteq u$ by the choice of sets $X$ and $u$. Therefore, $\phi\notin u$ because set $u$ is consistent.
\end{proof}

\subsection{Properties of Modality $\K$}

\begin{lemma}\label{K child all}
For any agent $a\in \mathcal{A}$, any worlds $w,u\in W$, and any formula $\phi\in\Phi$, if $\K_a\phi\in w$ and $w\sim_a u$, then $\phi\in u$.
\end{lemma}
\begin{proof}
Assumptions $\K_a\phi\in w$ and $w\sim_a u$ imply $\phi\in u$ by Definition~\ref{canonical sim}.
\end{proof}

\begin{lemma}\label{K child exists}
For any agent $a\in \mathcal{A}$, any world $w\in W$, and any formula $\phi\in\Phi$, if $\K_a\phi\notin w$, then there is a world $u\in W$ such that $w\sim_a u$ and $\phi\notin u$.
\end{lemma}
\begin{proof}
First, we show that the following set of formulae is consistent
\begin{equation}\label{july16-X}
   X=\{\neg\phi\}\cup\{\psi\;|\;\K_a\psi\in w\}\cup\{\chi\;|\;\N\chi\in X_0\}. 
\end{equation}
Assume the opposite. Then, there are formulae 
\begin{equation}\label{july16-psi}
    \K_a\psi_1,\dots,\K_a\psi_m\in w
\end{equation}
and formulae 
\begin{equation}\label{july16-chi}
    \N\chi_1,\dots,\N\chi_n\in X_0
\end{equation}
such that
$$
\psi_1,\dots,\psi_m,\chi_1,\dots,\chi_n\vdash \phi.
$$
Thus, by Lemma~\ref{superdistributivity lemma},
$$
\K_a\psi_1,\dots,\K_a\psi_m,\K_a\chi_1,\dots,\K_a\chi_n\vdash \K_a\phi.
$$
Hence, by assumption~(\ref{july16-psi}),
\begin{equation}\label{june29}
    w,\K_a\chi_1,\dots,\K_a\chi_n\vdash \K_a\phi.
\end{equation}
Consider any integer $i\le n$. Note that $\N\chi_i\to\K_a\chi_i$ is an instance of the Knowledge of Necessity axiom. Then, $\vdash \N(\N\chi_i\to\K_a\chi_i)$ by the Necessitation inference rule. Thus, $\vdash \N\N\chi_i\to\N\K_a\chi_i$ by the Distributivity axiom and the Modus Ponens inference rule. Note that $\vdash \N\phi\to\N\N\phi$ by Lemma~\ref{positive introspection lemma}. Hence, $\vdash \N\chi_i\to\N\K_a\chi_i$ by the laws of propositional reasoning. Then, $X_0\vdash \N\K_a\chi$ by assumption~(\ref{july16-chi}). Thus, $\N\K_a\chi\in X_0$ because set $X_0$ is maximal. Hence, $\K_a\chi_i\in w$ by Definition~\ref{canonical world definition} for any integer $i\le n$. Then, statement~(\ref{june29}) implies that $w\vdash \K_a\phi$. Thus, $\K_a\phi\in w$ because set $w$ is maximal, which contradicts assumption $\K_a\phi\notin w$ of the lemma. Therefore, set $X$ is consistent.

By Lemma~\ref{Lindenbaum's lemma}, set $X$ can be extended to a maximal consistent set $u$. Then, $\{\chi\;|\;\N\chi\in X_0\}\subseteq X\subseteq u$ by equation~(\ref{july16-X}). Thus, $u\in W$ by Definition~\ref{canonical world definition}. Also, $\{\psi\;|\;\K_a\psi\in w\}\subseteq X\subseteq u$ by equation~(\ref{july16-X}). Hence, $w\sim_a u$ by Definition~\ref{canonical sim}. Finally, $\neg\phi\in X\subseteq u$ also by equation~(\ref{july16-X}). Therefore, $\phi\notin u$ because set $u$ is consistent.
\end{proof}

\subsection{Common Properties of Modalities $\H$ and $\S$}

Recall that by $\E$ we denote one of the two emotional modalities: $\H$ and $\S$.

\begin{lemma}
For any agent $a\in \mathcal{A}$, any worlds $w,u\in W$, and any formula $\phi\in\Phi$, if $\E_a\phi\in w$ and $w\sim_a u$, then $\phi\in u$. 
\end{lemma}
\begin{proof}
By Lemma~\ref{E to K} and the Modus Ponens inference rule, the assumption $\E_a\phi\in w$ implies $w\vdash\K_a\phi$. Thus, $\K_a\phi\in w$ because set $w$ is maximal. Therefore, $\phi\in u$ by Lemma~\ref{K child all} and the assumption $w\sim_a u$.
\end{proof}

\begin{lemma}
For any world $w\in W$, and any formula $\E_a\phi\in w$, there is a world $u\in W$ such that $\phi\notin u$. 
\end{lemma}
\begin{proof}
By the Counterfactual axiom and the Modus Ponens inference rule, assumption $\E_a\phi\in w$ implies  $w\vdash \neg\N\phi$. Thus, $\N\phi\notin w$ because set $w$ is consistent. Therefore, by Lemma~\ref{N child exists}, there is a world $u\in W$ such that $\phi\notin u$. 
\end{proof}

\subsection{Properties of Modality $\H$}

\begin{lemma}\label{H child all}
For any agent $a\in \mathcal{A}$, any worlds $w,u,u'\in W$, and any formula $\phi\in\Phi$, if $\H_a\phi\in w$, $\phi\notin u$ and $\phi\in u'$, then $u\prec_a u'$.
\end{lemma}
\begin{proof}
The assumption $\H_a\phi\in w$ implies $\neg\H_a\phi\notin w$ because set $w$ is consistent. Thus, $\N\neg\H_a\phi\notin X_0$ by Definition~\ref{canonical world definition} because $w\in W$. Hence, $\neg\N\neg\H_a\phi\in X_0$ because set $X_0$ is maximal. Then, $\cN\H_a\phi\in X_0$ by the definition of modality $\cN$. Thus, $\phi\in \Delta_a$ by Definition~\ref{Delta a}. Therefore, 
$u\prec_a u'$ by Definition~\ref{canonical prec} and the assumptions $\phi\notin u$ and $\phi\in u'$ of the lemma. 
\end{proof}

\begin{lemma}\label{H children exist}
For any agent $a\in \mathcal{A}$, any world $w\in W$, and any formula $\phi\in\Phi$, if $\H_a\phi\notin w$, $\K_a\phi\in w$, and $\N\phi\notin w$, then there are worlds $u,u'\in W$ such that $\phi\notin u$, $\phi\in u'$, and  $u\not\prec_a u'$.
\end{lemma}
\begin{proof}
By Lemma~\ref{sqsubseteq total pre-order lemma}, relation $\sqsubseteq$ forms a total pre-order on set $\Delta_a$. By Lemma~\ref{Delta_a is countable}, set $\Delta_a$ is countable. Thus, by the axiom of countable choice, there is an ordering of all formulae in set $\Delta_a$ that agrees with pre-order $\sqsubseteq$. Generally speaking, such ordering is not unique. We fix any such ordering:
\begin{equation}\label{delta chain}
    \delta_0\sqsubseteq\delta_1\sqsubseteq\delta_2\sqsubseteq\delta_3\sqsubseteq\dots
\end{equation}
If set $\Delta_a$ is finite, the above ordering has some finite number $n$ of elements. In this case, the ordering is isomorphic to ordinal $n$. Otherwise, it is isomorphic to ordinal $\omega$. Let $\alpha$ be the ordinal which is the type of ordering~(\ref{delta chain}).  Ordinal $\alpha$ is either finite or is equal to $\omega$.

For any ordinal $k\le \alpha$, we consider set 
\begin{equation}\label{Yk definition}
    Y_k=\{\phi\}\cup\{\neg\delta_i\;|\; i<k\}\cup \{\psi\;|\;\N\psi\in X_0\}.
\end{equation}

\begin{claim}\label{trans induction lemma}
If there is no finite ordinal $k<\alpha$ such that $Y_k$ is consistent and $Y_{k+1}$ is inconsistent, then $Y_\alpha$ is consistent.
\end{claim}
\begin{proof-of-claim}
To prove that $Y_\alpha$ is consistent, it suffices to show that $Y_k$ is consistent for each ordinal $k\le \alpha$. We prove this by transfinite induction.

\vspace{2mm}
\noindent{\bf Zero Case:} Suppose that $Y_0$ is not consistent. Thus, there are formulae $\N\psi_1,\dots,\N\psi_n\in X_0$ such that
$
\psi_1,\dots,\psi_n\vdash \neg\phi
$. Hence, 
$
\N\psi_1,\dots,\N\psi_n\vdash \N\neg\phi
$
by Lemma~\ref{superdistributivity lemma}. 
Then,
$
X_0\vdash \N\neg\phi
$
by the assumption $\N\psi_1,\dots,\N\psi_n\in X_0$. Thus,
because set $X_0$ is maximal,
\begin{equation}\label{zero case point}
  \N\neg\phi\in X_0.  
\end{equation}
Hence, 
$\neg\phi\in w$,
by Definition~\ref{canonical world definition}.
Then,
$w\vdash \neg\K_a\phi$
by the contraposition of the Truth axiom and 
propositional reasoning.
Therefore, 
$\K_a\phi\notin w$ because set $w$ is consistent, which contradicts the assumption 
$\K_a\phi\in w$ of lemma.

\vspace{2mm}
\noindent{\bf Successor Case:} Suppose that set $Y_k$ is consistent for some $k<\alpha$. By the assumption of the claim, there is no finite ordinal $k<\alpha$ such that $Y_k$ is consistent and $Y_{k+1}$ is inconsistent. Therefore, set $Y_{k+1}$ is consistent.

\vspace{2mm}
\noindent{\bf Limit Case:} Suppose that set $Y_\omega$ is not consistent. Thus, there are formulae $\N\psi_1,\dots,\N\psi_n\in X_0$ and some finite $k<\omega$ such that
$
\psi_1,\dots,\psi_n,\neg\delta_0,\dots,\neg\delta_{k-1}\vdash \neg\phi
$.
Therefore, set $Y_k$ is not consistent.
\end{proof-of-claim}

Let $k'$ be any finite ordinal $k'<\alpha$ such that $Y_{k'}$ is consistent and $Y_{k'+1}$ is inconsistent. If such finite ordinal does not exist, then let $k'$ be ordinal $\alpha$. Note that in either case, set $Y_{k'}$ is consistent by Claim~\ref{trans induction lemma}. Let $u'$ be any maximal consistent extension of set $Y_{k'}$. Such extension exists by Lemma~\ref{Lindenbaum's lemma}. Note that $\phi\in Y_{k'}\subseteq u'$ by equation~(\ref{Yk definition}) and the choice of set $u'$.

\begin{claim}\label{u' in W claim}
$u'\in W$.
\end{claim}
\begin{proof-of-claim}
Consider any formula $\N\psi\in X_0$. By Definition~\ref{canonical world definition}, it suffices to show that $\psi\in u'$. Indeed, $\psi\in Y_{k'}$ by equation~(\ref{Yk definition})  and the assumption $\N\psi\in X_0$. Thus, $\psi\in u'$ because $Y_{k'}\subseteq u'$ by the choice of set $u'$.
\end{proof-of-claim}

Consider the following set of formulae:
\begin{equation}\label{choice of Z equation}
   Z=\{\neg\phi\}\cup\{\delta_i\;|\;k'\le i<\alpha\}\cup \{\psi\;|\;\N\psi\in X_0\}.
\end{equation}

\begin{claim}\label{Z claim}
Set $Z$ is consistent.
\end{claim}
\begin{proof-of-claim}
We consider the following two cases separately:

\noindent{\em Case 1:} $k'<\alpha$. Suppose that set $Z$ is not consistent. Thus, there are finite ordinals $m<\alpha$ and $n<\omega$ and formulae
\begin{equation}\label{choice of psi}
    \N\psi_1,\dots\N\psi_n\in X_0
\end{equation}
such that $k'\le m$ and
$$
\delta_{k'},\delta_{k'+1},\dots,\delta_{m},\psi_1,\dots,\psi_n\vdash \phi.
$$
Hence, by the Modus Ponens inference rule applied $m-k'$ times,
$$
\delta_{k'},\delta_{k'}\to\delta_{k'+1},
\delta_{k'+1}\to\delta_{k'+2},
\delta_{k'+2}\to\delta_{k'+3},
\dots,
\delta_{m-1}\to\delta_{m},\psi_1,\dots,\psi_n\vdash \phi.
$$
Then, by Lemma~\ref{deduction lemma},
$$
\delta_{k'}\to\delta_{k'+1},
\delta_{k'+1}\to\delta_{k'+2},
\dots,
\delta_{m-1}\to\delta_{m},\psi_1,\dots,\psi_n\vdash \delta_{k'}\to\phi.
$$
Thus, by Lemma~\ref{superdistributivity lemma},
$$
\N(\delta_{k'}\to\delta_{k'+1}),
\N(\delta_{k'+1}\to\delta_{k'+2}),
\dots,
\N(\delta_{m-1}\to\delta_{m}),\N\psi_1,\dots,\N\psi_n\vdash \N(\delta_{k'}\to\phi).
$$
Recall that 
$\delta_{k'}\sqsubseteq\delta_{k'+1}\sqsubseteq\dots \sqsubseteq\delta_{m}$ by assumption~(\ref{delta chain}). Hence, it follows that $\N(\delta_{k'}\to\delta_{k'+1}), \N(\delta_{k'+1}\to\delta_{k'+2}), \dots, \N(\delta_{m-1}\to\delta_{m})\in X_0$ by Definition~\ref{sqsubseteq definition}. Then,
$$
X_0,\N\psi_1,\dots,\N\psi_n\vdash \N(\delta_{k'}\to\phi).
$$
Thus, by assumption~(\ref{choice of psi}), 
\begin{equation}\label{N delta to phi}
   X_0\vdash \N(\delta_{k'}\to\phi). 
\end{equation}
At the same time, $k'<\alpha$ by the assumption of the case. Hence, set $Y_{k'+1}$ is not consistent by the choice of the finite ordinal $k'$. Then, by equation~(\ref{Yk definition}), there must exist formulae
\begin{equation}\label{choice of psi prime}
    \N\psi'_1,\dots,\N\psi'_p\in X_0
\end{equation}
such that
$$
\neg\delta_{0},\neg\delta_{1},\neg\delta_{2},\dots,\neg\delta_{k'},\psi'_1,\dots,\psi'_p\vdash \neg\phi.
$$
In other words,
$$
\neg\delta_{k'},\neg\delta_{k'-1},\neg\delta_{k'-2},\dots,\neg\delta_{0},\psi'_1,\dots,\psi'_p\vdash \neg\phi.
$$
Thus, by applying the Modus Ponens inference rule $k'$ times,
$$
\neg\delta_{k'}, \neg\delta_{k'}\to\neg\delta_{k'-1},\neg\delta_{k'-1}\to \neg\delta_{k'-2},\dots,\neg\delta_{1}\to\neg\delta_{0},\psi'_1,\dots,\psi'_p\vdash \neg\phi.
$$
Hence, by Lemma~\ref{deduction lemma},
$$
\neg\delta_{k'}\to\neg\delta_{k'-1},\neg\delta_{k'-1}\to \neg\delta_{k'-2},\dots,\neg\delta_{1}\to\neg\delta_{0},\psi'_1,\dots,\psi'_n\vdash \neg\delta_{k'}\to \neg\phi.
$$
Then, by Lemma~\ref{superdistributivity lemma},
\begin{eqnarray*}
&&\N(\neg\delta_{k'}\to\neg\delta_{k'-1}),\N(\neg\delta_{k'-1}\to \neg\delta_{k'-2}),\dots,\N(\neg\delta_{1}\to\neg\delta_{0}),\\
&&\hspace{65mm}\N\psi'_1,\dots,\N\psi'_n\vdash \N(\neg\delta_{k'}\to \neg\phi).
\end{eqnarray*}
Recall that 
$\delta_{0}\sqsubseteq\delta_{1}\sqsubseteq\dots \sqsubseteq\delta_{k'}$ by assumption~(\ref{delta chain}). Thus, it follows that $\N(\neg\delta_{1}\to\neg\delta_{0}), \N(\neg\delta_{2}\to\neg\delta_{1}), \dots, \N(\neg\delta_{k'}\to\neg\delta_{k'-1})\in X_0$ by Lemma~\ref{sqsubseteq contraposition lemma}. Hence,
$$
X_0,\N\psi'_1,\dots,\N\psi'_n\vdash \N(\neg\delta_{k'}\to\neg\phi).
$$
Then, by assumption~(\ref{choice of psi prime}),
\begin{equation}\label{N neg delta to neg phi}
    X_0\vdash \N(\neg\delta_{k'}\to\neg\phi).
\end{equation}
Thus, by item 1 of Lemma~\ref{N biconditional lemma} and statement~(\ref{N delta to phi}),
\begin{equation}\label{N biconditional statement}
    X_0\vdash \N(\delta_{k'}\leftrightarrow\phi).
\end{equation}
Note that $\delta_{k'}\in\Delta_a$ because (\ref{delta chain}) is an ordering of set $\Delta_a$. Hence, by Definition~\ref{Delta a}, either $\cN\H_a\delta_{k'}\in X_0$ or $\cN\S_a\neg\delta_{k'}\in X_0$. Then, by items 1 and 2 of Lemma~\ref{substitution lemma} and statement~(\ref{N biconditional statement}), either $X_0\vdash \cN\H_a\phi$ or $X_0\vdash \cN\S_a\neg\phi$. Thus, either $X_0\vdash \N\cN\H_a\phi$ or $X_0\vdash \N\cN\S_a\neg\phi$ by the definition of modality $\cN$, the Negative Introspection axiom, and the Modus Ponens inference rule. Hence, either $\N\cN\H_a\phi\in X_0$ or $ \N\cN\S_a\neg\phi\in X_0$ because set $X_0$ is maximal. Then, either $\cN\H_a\phi\in w$ or $\cN\S_a\neg\phi\in w$ by Definition~\ref{canonical world definition} and assumption $w\in W$ of the lemma. Thus, $w\vdash \K_a\phi\to\H_a\phi$ by the first Emotional Predictability axiom and propositional reasoning. Hence, $w\vdash \H_a\phi$ by assumption $\K_a\phi$ of the lemma and the Modus Ponens inference rule. Therefore, $\H_a\phi\in w$ because set $w$ is maximal, which contradicts assumption $\H_a\phi\notin w$ of the lemma.

\noindent{\em Case 2:} $\alpha\le k'$. Recall that $k'\le\alpha$ by the choice of ordinal $k'$ made after the end of the proof of Claim~\ref{trans induction lemma}. Thus, $k'=\alpha$. Hence, $Z=\{\neg\phi\}\cup \{\psi\;|\;\N\psi\in X_0\}$ by equation~(\ref{choice of Z equation}). Then, inconsistency of set $Z$ implies that there are formulae 
\begin{equation}\label{choice of psi 1 psi n}
    \N\psi_1,\dots\N\psi_n\in X_0
\end{equation}
such that
$
\psi_1,\dots, \psi_n\vdash \phi
$.
Thus,
$
\N\psi_1,\dots, \N\psi_n\vdash \N\phi
$
by Lemma~\ref{superdistributivity lemma}. Hence,
\begin{equation}\label{split point}
   X_0\vdash \N\phi 
\end{equation}
by the assumption~(\ref{choice of psi 1 psi n}).
Then,
$
X_0\vdash \N\N\phi
$
by Lemma~\ref{positive introspection lemma} and the Modus Ponens inference rule. 
Hence,
$\N\N\phi\in X_0$
because set $X_0$ is maximal.
Therefore,
$
\N\phi\in w
$
by Definition~\ref{canonical world definition}, which contradicts the assumption of the lemma.
\end{proof-of-claim}
Let $u$ be any maximal consistent extension of set $Z$. Note that $\neg\phi\in Z\subseteq u$ by equation~(\ref{choice of Z equation}) and the choice of set $u$.

\begin{claim}\label{u W claim}
$u\in W$.
\end{claim}
\begin{proof-of-claim}
Consider any formula $\N\psi\in X_0$. By Definition~\ref{canonical world definition}, it suffices to show that $\psi\in u$. Indeed, $\psi\in Z$ by equation~(\ref{choice of Z equation})  and the assumption $\N\psi\in X_0$. Thus, $\psi\in u$ because $Z\subseteq u$ by the choice of set $u$.
\end{proof-of-claim}

\begin{claim}\label{u prec u' claim}
$u\not\prec_a u'$.
\end{claim}
\begin{proof-of-claim}
Suppose that $u\prec_a u'$. Thus, by Definition~\ref{canonical prec}, there is a formula $\delta\in\Delta_a$ such that $\delta\notin u$ and $\delta\in u'$. Recall that (\ref{delta chain}) is an ordering of the set $\Delta_a$. Hence, there must exist an integer $i<\alpha$ such that 
\begin{equation}\label{about u and u'}
    \delta_i\notin u \mbox{ and }\delta_i\in u'.
\end{equation}
We consider the following two cases separately:

\noindent{\em Case 1:} $i< k'$. Then, $\neg\delta_i\in Y_{k'}\subseteq u'$ by equation~(\ref{Yk definition}) and the choice of set $u'$. Thus,  $\delta_i\notin u'$ because set $u'$ is consistent, which contradicts to statement~(\ref{about u and u'}).

\noindent{\em Case 2:} $k'\le i$. Then, $\delta_i\in Z\subseteq u$ by equation~(\ref{choice of Z equation}) and the choice of set $u$, which contradicts to statement~(\ref{about u and u'}).
\end{proof-of-claim}
This concludes the proof of the lemma.
\end{proof}

\subsection{Properties of Modality $\S$}

\begin{lemma}\label{S child all}
For any agent $a\in \mathcal{A}$, any worlds $w,u,u'\in W$, and any formula $\phi\in\Phi$, if $\S_a\phi\in w$, $\phi\in u$ and $\phi\notin u'$, then $u\prec_a u'$.
\end{lemma}
\begin{proof}
Note that $\phi\leftrightarrow\neg\neg\phi$ is a propositional tautology. Thus, $\vdash\N(\phi\leftrightarrow\neg\neg\phi)$ by the Necessitation inference rule. Hence, $\vdash\S_a\phi\to\S_a\neg\neg\phi$ by the Substitution axiom and the Modus Ponens inference rule. Then, $w\vdash \S_a\neg\neg\phi$ by the Modus Ponens inference rule and the assumption $\S_a\phi\in w$ of the lemma. Thus, $\neg\S_a\neg\neg\phi\notin w$ because set $w$ is consistent. Hence, $\N\neg\S_a\neg\neg\phi\notin X_0$ by Definition~\ref{canonical world definition}. Then, $\neg\N\neg\S_a\neg\neg\phi\in X_0$ because set $X_0$ is maximal. Thus, $\cN\S_a\neg\neg\phi\in X_0$ by the definition of modality $\cN$. Hence, $\neg\phi\in \Delta_a$ by Definition~\ref{Delta a}. Therefore, $u\prec_a u'$ by Definition~\ref{canonical prec} and the assumptions $\neg\phi\notin u$ and  $\neg\phi\in u'$ of the lemma.
\end{proof}

\begin{lemma}\label{S children exist}
For any agent $a\in \mathcal{A}$, any world $w\in W$, and any formula $\phi\in\Phi$, if $\S_a\phi\notin w$, $\K_a\phi\in w$, and $\N\phi\notin w$, then there are worlds $u,u'\in W$ such that $\phi\in u$, $\phi\notin u'$, and  $u\not\prec_a u'$.
\end{lemma}
\begin{proof}
The proof of this lemma is similar to the proof of Lemma~\ref{H children exist}. Here we outline the differences. The choice of ordering~(\ref{delta chain}) and  of ordinal $\alpha$ remains the same. Sets $Y_k$ for any ordinal $k\le \alpha$ is now defined as
\begin{equation}\label{Yk definition S}
    Y_k=\{\neg\phi\}\cup\{\neg\delta_i\;|\; i<k\}\cup \{\psi\;|\;\N\psi\in X_0\}.
\end{equation}
This is different from equation~(\ref{Yk definition}) because set $Y_k$ now includes formula $\neg\phi$ instead of formula $\phi$. 

The statement of Claim~\ref{trans induction lemma} remains the same. The proof of this claim is also the same except for the Zero Case. In the Zero Case, the proof is similar to the original till equation~(\ref{zero case point}). Because set $Y_k$ now contains formula $\neg\phi$ instead of formula $\phi$, equation~(\ref{zero case point}) will now have the form $\N\phi\in X_0$. Thus, in our case, $X_0\vdash\N\N\phi$ by Lemma~\ref{positive introspection lemma} and the Modus Ponens inference rule. Hence, $\N\N\phi\in X_0$ because set $X_0$ is maximal. Then, $\N\phi\in w$ by Definition~\ref{canonical world definition}, which contradicts the assumption $\N\phi\notin w$ of the lemma. 

The statement and the proof of Claim~\ref{u' in W claim} remain the same. Set $Z$ will now be defined as
\begin{equation}\label{new choice of Z equation}
   Z=\{\phi\}\cup\{\delta_i\;|\;k'\le i<\alpha\}\cup \{\psi\;|\;\N\psi\in X_0\}.
\end{equation}
This is different from equation~(\ref{choice of Z equation}) because set $Z$ now includes formula $\phi$ instead of formula $\neg\phi$. 

The statement of Claim~\ref{Z claim} remains the same. The Case 1 of the proof of this case is similar to the original proof of Claim~\ref{Z claim} till formula~(\ref{N neg delta to neg phi}), except for formula $\phi$ will be used instead of formula $\neg\phi$ and formula $\neg\phi$ instead of formula $\phi$ everywhere in that part of the proof. In particular formula~(\ref{N delta to phi}) will now have the form 
\begin{equation}\label{new N delta to phi}
   X_0\vdash \N(\delta_{k'}\to\neg\phi). 
\end{equation}
and formula~(\ref{N neg delta to neg phi}) will now have the form 
\begin{equation}\label{new N neg delta to neg phi}
    X_0\vdash \N(\neg\delta_{k'}\to\phi).
\end{equation}
The argument after formula~(\ref{N neg delta to neg phi}) will change as follows.
By item 2 of Lemma~\ref{N biconditional lemma} and statements ~(\ref{new N delta to phi}) and (\ref{new N neg delta to neg phi}),
\begin{equation}\label{N biconditional statement}
    X_0\vdash \N(\delta_{k'}\leftrightarrow\neg\phi).
\end{equation}
Note that $\delta_{k'}\in\Delta_a$ because (\ref{delta chain}) is an ordering of set $\Delta_a$. Hence, by Definition~\ref{Delta a}, either $\cN\H_a\delta_{k'}\in X_0$ or $\cN\S_a\neg\delta_{k'}\in X_0$.
Then, by items 1 or 3 of Lemma~\ref{substitution lemma} and statement~(\ref{N biconditional statement}), either $X_0\vdash \cN\H_a\neg\phi$ or $X_0\vdash \cN\S_a\phi$. 
Thus, either $X_0\vdash \N\cN\H_a\neg\phi$ or $X_0\vdash \N\cN\S_a\phi$ by the definition of modality $\cN$, the Negative Introspection axiom, and the Modus Ponens inference rule. 
Hence, either $\N\cN\H_a\neg\phi\in X_0$ or $ \N\cN\S_a\phi\in X_0$ because set $X_0$ is maximal.
Then, either $\cN\H_a\neg\phi\in w$ or $\cN\S_a\phi\in w$ by Definition~\ref{canonical world definition} and assumption $w\in W$ of the lemma. 
Thus, $w\vdash \K_a\phi\to\S_a\phi$ by the second Emotional Predictability axiom and propositional reasoning. 
Hence, $w\vdash \S_a\phi$ by assumption $\K_a\phi$ of the lemma and the Modus Ponens inference rule. Therefore, $\S_a\phi\in w$ because set $w$ is maximal, which contradicts assumption $\S_a\phi\notin w$ of the lemma.

The Case 2 of the proof of Claim~\ref{Z claim} will be similar to the original proof of Claim~\ref{Z claim} till formula~(\ref{split point}) except that formula except for formula $\phi$ will be used instead of formula $\neg\phi$ and formula $\neg\phi$ instead of formula $\phi$ everywhere in that part of the proof. Statement~(\ref{split point}) will now have the form $X_0\vdash \N\neg\phi$. From this point, the proof will continue as follows. Statement  $X_0\vdash \N\neg\phi$ implies that $\N\neg\phi\in X_0$ because set $X_0$ is maximal. Then, $\neg\phi\in w$ by Definition~\ref{canonical world definition}. Hence, $w\vdash\neg\K_a\phi$ by the contraposition of the Truth axiom. Therefore, $\K_a\phi\notin w$ because set $w$ is consistent, which contradicts the assumption  $\K_a\phi\in w$ of the lemma.

The statements and the proofs of Claim~\ref{u W claim} and Claim~\ref{u prec u' claim} remain the same as in the original proof.
\end{proof}

\subsection{Final Steps}

We are now ready to state and to prove the ``induction'' or ``truth'' lemma.

\begin{lemma}\label{induction lemma}
$w\Vdash \phi$ iff $\phi\in w$.
\end{lemma}
\begin{proof}
We prove the lemma by structural induction on formula $\phi$. If $\phi$ is a propositional variable, then the required follows from Definition~\ref{canonical pi} and item 1 of Definition~\ref{sat}. If formula $\phi$ is a negation or an implication, then the statement of the lemma follows from the induction hypothesis using items 2 and 3 of Definition~\ref{sat} and the maximality and the consistency of the set $w$ in the standard way.

Suppose that formula $\phi$ has the form $\K_a\psi$.

\vspace{1mm}
\noindent$(\Leftarrow):$ By Lemma~\ref{K child all}, assumption $\K_a\psi\in w$ implies that $\psi\in u$ for any world $u\in W$ such that $w\sim_a u$. Thus, by the induction hypothesis, $u\Vdash \psi$ for any world $u\in W$ such that $w\sim_a u$. Therefore, $w\Vdash\K_a\psi$ by item 5 of Definition~\ref{sat}.

\vspace{1mm}
\noindent$(\Rightarrow):$  Assume that $\K_a\psi\notin w$. Thus, by Lemma~\ref{K child exists}, there is a world $u\in W$ such that $w\sim_a u$ and $\psi\notin u$. Hence, $u\nVdash\psi$ by the induction hypothesis. Therefore, $w\nVdash\K_a\psi$ by item 5 of Definition~\ref{sat}.

If formula $\phi$ has the form $\N\psi$, then the proof is similar to the case $\K_a\psi$ except that Lemma~\ref{N child all} and Lemma~\ref{N child exists} are used instead of Lemma~\ref{K child all} and Lemma~\ref{K child exists} respectively. Also, item 4 of Definition~\ref{sat} is used instead of item 5.

Assume that formula $\phi$ has the form $\H_a\psi$.

\vspace{1mm}
\noindent$(\Leftarrow):$ Assume that $\H_a\psi\in w$. To prove that $w\Vdash\H_a\psi$, we verify conditions (a), (b), and (c) from item 6 of Definition~\ref{sat}.

\begin{enumerate}
    \item[(a)] The assumption $\H_a\psi\in w$ implies $w\vdash\K_a\psi$ by Lemma~\ref{E to K} and the Modus Ponens inference rule. Thus, $\K_a\psi\in w$ because set $w$ is maximal. Hence, by Lemma~\ref{K child all}, for any world $u\in W$, if $\psi\in u$, then $w\sim_a u$ . Therefore, by the induction hypothesis, $u\Vdash \psi$ for any world $u\in W$ such that $w\sim_a u$.
    \item[(b)] By Lemma~\ref{H child all}, assumption $\H_a\psi\in w$ implies that for any worlds $u,u'\in W$, if $\psi\notin u$ and $\psi\in u'$, then $u\prec_a u'$. Thus, by the induction hypothesis, for any worlds $u,u'\in W$, if $u\nVdash\psi$ and $u'\Vdash\psi$, then $u\prec_a u'$.
    \item [(c)] By the Counterfactual axiom and the Modus Ponens inference rule, the assumption $\H_a\psi\in w$ implies that $w\vdash\neg\N\psi$. Thus, $\N\psi\notin w$ because set $w$ is consistent.  Hence, by Lemma~\ref{N child exists}, there is a world $u\in W$ such that $\psi\notin u$. Therefore, $u\nVdash\psi$ by the induction hypothesis.
\end{enumerate}

\vspace{1mm}
\noindent $(\Rightarrow):$  Assume that $\H_a\psi\notin w$. We consider the following three cases separately:

\vspace{1mm}
\noindent{\bf Case I:} $\K_a\psi\notin w$. Thus, by Lemma~\ref{K child exists}, there is a world $u\in W$ such that $w\sim_a u$ and $\psi\notin u$. Hence, $u\nVdash\psi$ by the induction hypothesis. Therefore, $w\nVdash\H_a\psi$ by item 6(a) of Definition~\ref{sat}.

\vspace{1mm}
\noindent{\bf Case II:} $\N\psi\in w$. Then, $\psi\in u$ for any world $u\in W$  by Lemma~\ref{N child all}. Thus, by the induction hypothesis, $u\Vdash \psi$ for any world $u\in W$.  Therefore, $w\nVdash\H_a\psi$ by item 6(c) of Definition~\ref{sat}.

\vspace{1mm}
\noindent{\bf Case III:} $\K_a\psi\in w$ and $\N\psi\notin w$. Thus, by the assumption $\H_a\psi\notin w$ and Lemma~\ref{H children exist}, there are worlds $u,u'\in W$ such that $\psi\notin u$, $\psi\in u'$, and  $u\not\prec_a u'$. Hence, $u\nVdash\psi$ and $u'\Vdash\psi$ by the induction hypothesis. Therefore, $w\nVdash\H_a\psi$ by item 6(b) of Definition~\ref{sat}.

If formula $\phi$ has the form $\S_a\psi$, then the argument is similar to the one above, except that Lemma~\ref{S child all} and Lemma~\ref{S children exist} are used instead of Lemma~\ref{H child all} and Lemma~\ref{H children exist} respectively.
\end{proof}

\begin{theorem}[strong completeness]
If $X\nvdash \phi$, then there is world $w$ of an epistemic model with preferences such that $w\Vdash\chi$ for each formula $\chi\in X$ and  $w\nVdash \phi$.
\end{theorem}
\begin{proof}
The assumption $X\nvdash \phi$ implies that set $X\cup\{\neg\phi\}$ is consistent. Thus, by Lemma~\ref{Lindenbaum's lemma}, there is a maximal consistent set $w$ such that $X\cup\{\neg\phi\}\subseteq w$. Consider canonical epistemic model with preferences $M(w)$. By Lemma~\ref{X0 in W lemma}, set $w$ is one of the worlds of this model. Then, $w\Vdash \chi$ for any formula $\chi\in X$ and $w\Vdash \neg\phi$
by Lemma~\ref{induction lemma}. Therefore, $w\nVdash \phi$ by item 2 of Definition~\ref{sat}.
\end{proof}




\section{Conclusion}

In this article we proposed a formal semantics for happiness and sadness, proved that these two notions are not definable through each other and gave a complete logical system capturing the properties of these notions.
The approach to happiness that we advocated could be captured by famous saying ``Success is getting what you want, happiness is wanting what you get''. Although popular, this view is not the only possible. As we mentioned in the introduction, some view happiness as ``getting what you want''.  

As defined in this article, happiness and sadness are grounded in agent's knowledge. We think that an interesting next step could be exploring belief-based happiness and sadness. A framework for beliefs, similar to our epistemic models with preferences, has been proposed by Liu~\cite{l11springer}.


\bibliographystyle{elsarticle-num}  
\bibliography{sp}
\end{document}